\documentclass[11pt]{article}

\usepackage[letterpaper, margin=1in]{geometry}

\usepackage{amsmath, amssymb, amsthm}
\usepackage{tcolorbox}
\usepackage{wrapfig}
\usepackage{blkarray}
\usepackage{booktabs} 
\usepackage{multirow}
\usepackage{thm-restate}
\usepackage{thmtools}
\usepackage{subcaption}

\usepackage{bbm}
\usepackage[utf8]{inputenc} 
\usepackage[T1]{fontenc}    
\usepackage{hyperref}       
\usepackage{url}            
\usepackage{booktabs}       
\usepackage{amsfonts}       
\usepackage{nicefrac}       
\usepackage{microtype}      
\usepackage{xcolor}         

\usepackage[style=numeric, citestyle=numeric-comp, sorting=none, backend=biber, natbib=true, backref=true]{biblatex}

\newtheorem{assumption}{Assumption}
\newtheorem{definition}{Definition}
\newtheorem{theorem}{Theorem}
\newtheorem{proposition}[theorem]{Proposition}
\newtheorem{corollary}{Corollary}
\newtheorem{lemma}{Lemma}
\newtheorem{remark}{Remark}

\DeclareMathOperator*{\argmin}{arg\,min}

\DeclareMathOperator{\Tr}{Tr}

\newenvironment{proof*}
  {\par\pushQED{\qed}\normalfont\ignorespaces}
  {\popQED\par}

\title{
    On Regularization via Early Stopping for Least Squares Regression
}

\author{
    Rishi Sonthalia\thanks{
        Department of Mathematics, 
        Boston College, 
        \texttt{rishi.sonthalia@bc.edu}
    }
    \and
    Jackie Lok\thanks{
        ORFE Department, 
        Princeton University, 
        \texttt{jackie.lok@princeton.edu}
    }
    \and
    Elizaveta Rebrova\thanks{
        ORFE Department, 
        Princeton University, 
        \texttt{elre@princeton.edu}
    }
}

\date{}

\begin{document}

\maketitle

\begin{abstract}
    A fundamental problem in machine learning is understanding the effect of early stopping on the parameters obtained and the generalization capabilities of the model. Even for linear models, the effect is not fully understood for arbitrary learning rates and data. In this paper, we analyze the dynamics of discrete full batch gradient descent for linear regression. With minimal distributional assumptions, we characterize the trajectory of the parameters and the expected excess risk. Using this characterization, we show that when training with any learning rate schedule and finite time horizon, the early stopped solution is equivalent to the minimum norm solution for a generalized ridge regression problem. We also prove that early stopping is beneficial for generic data with arbitrary spectrum and for a wide variety of learning rate schedules. We provide an estimate for the optimal stopping time and empirically demonstrate the accuracy of our estimate.
\end{abstract}

\section{Introduction}

Early stopping is a commonly used method to regularize machine learning models, yet our understanding of the properties of models obtained via early stopping is far from complete. Recent work has shown that linear models trained with gradient descent can exhibit grokking, a phenomenon where the model initially overfits and generalizes poorly, only to later achieve better generalization after prolonged training~\cite{levi2024grokking}. As a result, the strategy of early stopping naturally raises a range of important questions, including: (a) \emph{What are the properties of early stopped models?} (b) \emph{How can we decide when to stop training?} (c) \emph{Under what circumstances is early stopping beneficial?}

Whether early stopping is beneficial or not, in terms of the out-of-sample generalization error, can depend on various parameters of the model. For example, for models trained by gradient descent, a crucial parameter is the choice of learning rates (or step sizes). Most prior works on early stopping for gradient descent consider constant step size schedules, frequently in the continuous-time gradient flow regime where the step sizes are assumed to be negligibly small~\cite{yao2007early, raskutti2014early}. However, the gradient flow approach does not explain how to quantify the optimal stopping time under more general and practical learning rate schedules.

Another known perspective views early stopping as inducing a form of $L_2$ regularization. Intuitively, limiting the number of training iterations ensures that the obtained parameters remain relatively close to their initialization. Hence, it is commonly believed that with zero initialization, early stopping induces $L_2$-type regularization, suggesting that the optimal stopping time should scale inversely with the minimum eigenvalue of the sample covariance matrix of the features~\cite{goodfellow2016deep, hu2022early, raskutti2014early}. However, formalizing this intuition has proven to be difficult.

Again, most prior works approach this by considering the continuous-time gradient flow dynamics for linear regression~\cite{SkGoBr1994} and with Gaussian assumptions on the feature matrix. For example, by studying the exact generalization dynamics under gradient flow, Advani et al.~\cite{AdSaSo2020} provide an estimate for stopping after $\Theta\big(\lambda_{\text{mode}}^{-1}\log(1 + \lambda_{\text{mode}})\big)$ iterations, where $\lambda_{\text{mode}}$ is the modal eigenvalue of the sample covariance matrix of the features.
Furthermore, Ali et al.~\cite{ali2019continuous} shows that gradient flow and ridge regression are tightly connected; namely, under the correspondence $\mu = 1/t$ between the ridge regularization parameter $\mu$ and time parameter $t$, the relative risk is always between $1$ and $1.6862$~\cite[Theorems 1 and 2]{ali2019continuous}. The authors also derive a formula for the risk of gradient flow under the Marchenko--Pastur limit with arbitrary covariance~\cite[Theorem 5]{ali2019continuous}.
Finally, \cite{shen2022optimal} provides descriptions of the optimal stopping time for gradient flow, up to constants, that hold with high probability under the more restrictive standard Gaussian model.
However, without the strong assumptions such as gradient flow and Gaussianity, questions such as \emph{when to stop} to optimize generalization error remain open~\cite{AdSaSo2020, ali2019continuous, shen2022optimal, xu2023towards}. 

In this paper, we do not impose the strong assumptions on the model as described above. Instead, we focus on the \emph{non-asymptotic discrete dynamics of gradient descent for linear regression}. This point of view allows us to develop an explicit formula for the parameters of our model after $k$ steps of gradient descent, denoted by $\beta_k$, which can be expressed in a closed form for many different learning rate schedules used in practice. We obtain several novel, more precise answers to the questions above using the explicit expression for $\beta_k$ and the associated generalization error. Specifically:
\begin{enumerate}
    \item \textbf{Exact trajectories.}
    We provide exact formulas for the discrete dynamics of $\beta_k$ obtained when solving the kernel ridge regression problem using gradient descent that make no assumptions on the data, learning rate schedule, or noise distribution (Proposition~\ref{prop:betak}). We state simple expressions for many common learning rate schedules.
    
    \item \textbf{Equivalence to generalized ridge regression.}
    For the ridgeless case, we show that for generic data, learning rate schedules, and stopping time $T$, the solution obtained after $T$ iterations is equivalent to the minimum norm solution for a generalized ridge regression problem (Theorem~\ref{thm:equivalence1}). Additionally, we show that any minimum norm solution to the ridge regression problem can also be obtained via early stopping if we can pick a distinct learning rate in each eigenspace of the sample covariance matrix of the features (Theorem~\ref{thm:equivalence2}).
    
    \item \textbf{Sufficient conditions for early stopping to be beneficial}.
    We provide sufficient conditions (Theorems~\ref{thm:earlystopping} and~\ref{thm:earlystopping-mu}) for when early stopping improves the generalization performance under some general assumptions. Conversely, we also provide sufficient conditions for when early stopping is not beneficial (Theorems~\ref{thm:earlystoppingConverse} and~\ref{thm:earlystopping-mu}). As a corollary, we show that early stopping is beneficial for many common learning rate schedules, independent of the input data distribution, and that it is not beneficial for some other learning rates (Remarks~\ref{rem:grokking} and~\ref{rem:grokking-ridge}).
    
    \item \textbf{Optimal stopping time estimate for ill-conditioned covariance and non-constant step sizes.}
    We propose an estimate for the optimal stopping time for generic data and a large class of learning rate schedules: see~\eqref{eq:when-to-stop-k} in Section~\ref{sec:time-estimate}. Our estimate generalizes a prior estimate from \cite{AdSaSo2020} that only considered constant step size schedules and requires the covariance matrix of the features to be well-conditioned. We numerically verify the accuracy of our estimate on synthetic and real datasets.
\end{enumerate}

\subsection{Related works}

The dynamics of gradient descent and its related variants have been the subject of extensive research in recent years. Without attempting to cover all the relevant literature, we will highlight some key works and approaches below that are most closely related to early stopping and our methodology.

\paragraph{Gradient flow dynamics.}
A common approach to understanding gradient-based methods is to study gradient flow. Prior works such as \cite{SkGoBr1994, AdSaSo2020, ali2019continuous, levi2024grokking, shen2022optimal, boursier2025early} have studied the gradient flow approximation of full batch gradient descent. Other works have also studied the dynamics of stochastic gradient descent or flow~\cite{ali2020sgf, paquette2025homogenization, paquette2022sgd}. Here, we study the discrete dynamics of full batch gradient descent instead of gradient flow, which, in particular, allows us to derive nontrivial results when non-constant step sizes are used.

\paragraph{Generalization error and regularization.}
In this paper, we study the regularization effects of early stopping in terms of the out-of-sample generalization error. There has been a lot of recent work that uses tools from random matrix theory to understand the generalization performance of linear regression with other forms of regularization. Some works studying generalization for the ridge regression problem include~\cite{DobribanWager2018, HastieEtAl2022, kobak2020ridge, yilmaz2022regularization, nakkiran2021optimal, wang2024near, JacotEtAl2020, jacot2020implicit}. One of the surprising results from such work is that the optimal ridge regularization parameter can be negative~\cite{yilmaz2022regularization}. Other works such as \cite{sonthalia2023training, kausiklowrank2024, dhifallah2021noise, cui2023high} study the effects of noise regularization. Recent work~\cite{sonthalia2023under} also studies the problem with both noise and ridge regularization. However, our focus is on the regularization effects of early stopping.

Concurrent work by Stark and Steinerberger~\cite{StarkSteinerberger2025} shows that if the ridge regularization parameter is large enough, then early stopping for gradient descent with constant step sizes is not beneficial, which we have established independently for general step size schedules (see Theorem~\ref{thm:earlystopping-mu} and Remark~\ref{rem:grokking-ridge}). In addition, \cite{StarkSteinerberger2025} elaborates on and proves optimality results for a fully data-driven methodology for estimating the optimal ridge regularization parameter.

\paragraph{Early stopping for statistical inverse problems.}
Early stopping of iterative methods has also been studied in the inverse-problems literature~\cite{blanchard2018early, blanchard2018optimal, hucker2025early, miftachov2025early}, where regularization is essential for stable recovery in ill-posed problems. In particular, early stopping for gradient descent applied to least-squares problems, known as the Landweber iteration in this literature, is analyzed in~\cite{blanchard2018optimal}. Many adaptive stopping rules in this area are based on the \emph{discrepancy principle}~\cite{EnglEtAl1996}, which terminates the iteration when the residual (i.e., training error) is of the same order as the noise level, which is assumed to be known a priori or able to be estimated.

Compared to our setting, the main distinction is that performance in this line of work is typically measured by the recovery error with respect to the underlying signal, or by a data-fidelity quantity such as the residual norm, whereas we focus on the out-of-sample generalization error. Thus, while discrepancy-based stopping rules provide a natural comparison, they are calibrated to a different objective from the one studied here. Moreover, in this literature, the model is typically situated in the underparameterized setting; by contrast, our analysis also applies in overparameterized settings, where the relation between interpolation, spectral regularization, and generalization is especially delicate.

\paragraph{Other related works.}
Understanding the dynamics of gradient descent for linear models can provide limited insights into the dynamics of gradient descent for neural networks. Specifically, \cite{chizat2019lazytraining} showed that in some circumstances, such as with a large variance initialization, the trajectory of the parameters of a neural network is close to the trajectory of a linearized model. This regime has been called \emph{lazy training}. Note that we can think of the linearized network as a kernel ridge regression problem with the neural tangent kernel~\cite{jacot2018ntk}. Recent work such as \cite{geiger2020disentangling} seeks to understand the relationship between feature learning and lazy training, and \cite{kumar2024grokking} shows that grokking for neural networks occurs due to the transition from lazy training to rich training. 

While our work focuses on studying gradient descent, studying the dynamics and regularization effects of early stopping for stochastic gradient descent (SGD) is an important related problem.
Different learning rate schedules have been analyzed for SGD, and we will only mention a few works that do so.
Learning rate schedules with linear decay and switching from constant to linear decay were studied in~\cite{gower2019sgd}. In the streaming setting, SGD with square root decaying step sizes and oblivious noise was investigated in~\cite{pesme2020online}, and SGD with exponentially decaying step sizes and semi-adversarial noise was analyzed in~\cite{jeong2025stochastic}.

\subsection{Setup and preliminaries}
We shall begin by specifying our model and notation.
Suppose that we are given $n$ training data points $(z_i, y_i)$ drawn i.i.d.\ from a distribution $\mathcal{D}$, where $z_i \in \mathbb{R}^d$ are the input vectors and $y_i \in \mathbb{R}$ are the responses. Let $\Psi: \mathbb{R}^d \to \mathbb{R}^p$ be a feature map and $x_i = \Psi(z_i)$.
Let $X \in \mathbb{R}^{n \times p}$ be the feature matrix containing the feature vectors $x_i$ of the training data as rows, and $y \in \mathbb{R}^{n}$ denote the response vector.
Furthermore, let $\beta_*$ be the Bayes optimal linear predictor defined by
\begin{equation}
    \beta_* := \argmin_{\beta \in \mathbb{R}^p} \mathbb{E}_{(z_1, y_1) \sim \mathcal{D}} \left[|y_1 - \Psi(z_1)^T\beta|^2\right],
\end{equation}
and define the residual $\varepsilon := y - X\beta_*$, which has i.i.d.\ coordinates $\varepsilon_i = y_i - x_i^T \beta_*$.

Given a regularization parameter $\mu \geq 0$, the kernel ridge regression problem aims to minimize the following loss:
\begin{equation} \label{eq:ls-ridge}
    L(\beta) = \frac{1}{2n}\|y - X\beta\|_2^2 + \frac{\mu}{2}\|\beta\|^2_2.
\end{equation}
We solve the ridge regression problem~\eqref{eq:ls-ridge} using gradient descent. Let $\beta_0 \in \mathbb{R}^p$ be the chosen initialization and $\{ \eta_k \}_{k \geq 1}$ denote the sequence of step sizes for each iteration. If we denote the iterate after $k$ iterations of gradient descent by $\beta_k \in \mathbb{R}^p$, then the gradient descent update for the $k$th iteration is given by
\begin{equation} \label{eq:update}
    \beta_{k} = \beta_{k-1} - \frac{\eta_{k}}{n} X^T(X\beta_{k-1} - y) - \eta_{k}\mu \beta_{k-1}.
\end{equation}

For any estimator $\beta \in \mathbb{R}^p$, the \emph{excess risk} with respect to $\beta_*$ is given by
\begin{equation} \label{eq:excess_risk}
    \mathcal{R}(\beta)
    := \mathbb{E}_{z \sim \mathcal{D}_{\mathrm{test}}}\left[\|\Psi(z)^T\beta - \Psi(z)^T\beta_*\|^2_2 \mid X \right]
    = \|\beta - \beta_*\|^2_\Sigma,
\end{equation}
where $\Sigma := \mathbb{E}_{z \sim \mathcal{D}_{\mathrm{test}}}[\Psi(z) \Psi(z)^T]$ is the (uncentered) covariance matrix of the feature vectors $x = \Psi(z)$, drawn according to some test distribution $z \sim \mathcal{D}_{\mathrm{test}}$, $\|\cdot\|_2$ denotes the Euclidean norm, $\|v\|_\Sigma^2 = v^T\Sigma v$, and the expectation is taken over a newly drawn test sample, conditional on the training features $X$.

Finally, we introduce the following notation. Let $\hat{\Sigma} := n^{-1} X^T X \in \mathbb{R}^{p \times p}$ be the sample (uncentered) covariance matrix of the features $X$.
Let $X = U\Sigma_X V^T$ and $X^TX = V \Lambda V^T$ be the singular value decomposition and eigendecomposition of $X$ and $X^T X$, respectively, where $U \in \mathbb{R}^{n \times n}$ and $V \in \mathbb{R}^{p \times p}$ are orthogonal matrices, and $\Sigma_X \in \mathbb{R}^{n \times p}$ and $\Lambda \in \mathbb{R}^{p \times p}$ only have non-zero values on their diagonals, which we may assume to be in non-increasing order. Note that $\Lambda = \Sigma_X^T \Sigma_X = n V^T \hat{\Sigma} V$.
We shall also be interested in representing the parameters $\beta_k$ and residual $\varepsilon$ in the eigenbasis of $X^T X$ (i.e., $V$) and $XX^T$ (i.e., $U$), respectively, which we denote by
\begin{equation} \label{eq:in-eigbasis}
    \tilde{\beta}_k := V^T\beta_k, \quad \tilde{\beta}_* := V^T\beta_*, \quad \text{and} \quad \check{\varepsilon} := U^T \varepsilon.
\end{equation}
We shall denote the identity matrix by $I$. Given a matrix $A$, we shall denote its Frobenius norm by $\|A\|_F = \sqrt{\mathrm{Tr}(A^T A)}$ and its Moore-Penrose pseudoinverse by $A^{\dagger}$.

\begin{remark}
We do not restrict the class of feature maps; for example, $\Psi$ could represent a random feature model~\cite{RahimiRecht2007}, a neural network, or the feature map for the neural tangent kernel~\cite{jacot2018ntk}. Moreover, we do not require that the distribution $\mathcal{D}_{\mathrm{test}}$ for $x = \Psi(z)$ at test time from~\eqref{eq:excess_risk} to be the same as the distribution $\mathcal{D}$ used to obtain the training data. So, our setup allows for \emph{covariate shift}, which is an important problem that has been widely studied; see, e.g.,~\cite{tripuraneni2021covariate, kausiklowrank2024}.
\end{remark}

\section{Exact trajectories}
To understand the regularization effect of early stopping and to determine when to stop, we begin by quantifying the trajectory of the parameters $\beta_k$ and the dynamics of the excess risk. The aim of this section is to prove Proposition~\ref{prop:betak}, which gives exact expressions for the dynamics of $\beta_k$ and its representation $\tilde{\beta}_k$ in the eigenbasis $V$.

To build intuition, we first consider the unregularized problem ($\mu = 0$) with constant step size $\eta_k \equiv \eta$ and zero initialization $\beta_0 = 0$. Simple calculations (e.g., see Chapter~7.8 of~\cite{goodfellow2016deep}) show that we can recursively unravel the gradient descent dynamics from~\eqref{eq:update} to obtain
\begin{equation} \label{eq:example_gd_formula_constant}
    \beta_k = \left(I - \frac{\eta}{n}X^TX\right)^k\beta_0 + \sum_{i=1}^{k}\frac{\eta}{n}\left(I - \frac{\eta}{n}X^TX\right)^{k-i}X^Ty.
\end{equation}
From~\eqref{eq:example_gd_formula_constant}, we can deduce the classical convergence result: if $\eta$ is small enough, then the first term tends to zero as $k \to \infty$, and the second term converges to the minimum norm solution $(X^T X)^{\dagger} X^T y$. Similarly, if a sequence of non-negative step sizes $\{ \eta_k \}_{k \geq 1}$ is used, then the above changes to 
\begin{equation} \label{eq:example_gd_formula}
    \beta_k = \prod_{i=1}^k\left(I - \frac{\eta_i}{n}X^TX\right)\beta_0 + \sum_{i=1}^k \frac{\eta_i}{n}\prod_{j=i+1}^k \left(I - \frac{\eta_j}{n}X^TX\right) X^T y.
\end{equation}

The general expression~\eqref{eq:example_gd_formula} involves an unfactored matrix polynomial with coefficients that depend on the step sizes. In the case of constant step sizes, the dynamics are tractable because the polynomial expression can be simplified using the formula for the geometric series.
One of the technical contributions of our paper is a technique for factoring the polynomial above that works for any step size schedule.
With this goal in mind, we shall define the following function $\phi$:

\begin{definition} \label{defn:phi}
Given a learning rate schedule $\{ \eta_i \}_{i \geq 1}$ and a real number $\zeta$, define the function $\phi(k; \zeta, \{ \eta_i \})$ for positive integers $k$ by
\[
    \phi(k;\zeta, \{\eta_i\}) := \phi(0;\zeta, \{\eta_i\}) \cdot \prod_{i=1}^k(1 - \eta_i \zeta),
\]
where $\phi(0; \zeta, \{\eta_i\})$ is a normalization constant. Whenever it is clear from context, we shall suppress the dependence of $\phi$ on $\zeta$ and $\{ \eta_i \}_{i \geq 1}$.
Furthermore, we define $\Phi(k; \mu, \Lambda) \in \mathbb{R}^{p \times p}$ to be the diagonal matrix whose $j$th diagonal entry is given by
\begin{equation} \label{eq:defn_Phi_matrix}
    \Phi_j(k; \mu, \Lambda) := \frac{\phi(k; \mu + n^{-1} \Lambda_j)}{\phi(0;\mu + n^{-1}\Lambda_j)},
\end{equation}
where $n^{-1} \Lambda_j$ is the $j$th largest eigenvalue of the sample covariance matrix $\hat{\Sigma} = n^{-1} X^T X = n^{-1} V \Lambda V^T$. Note that $\Phi_j(k; 0,\Lambda) = 1$ if $j > \mathrm{rank}(X)$. In matrix form, we can write
\begin{equation} \label{eq:defn_Phi_matrix2}
    \Phi(k; \mu, \Lambda) = \prod_{i=1}^k (I - \eta_i (\mu I + n^{-1} \Lambda)).
\end{equation}
Whenever it is clear from context, we shall also suppress the dependence of $\Phi$ on $\mu$ and $\Lambda$. 
\end{definition}

In Section~\ref{sec:phi-exact-expressions}, we will derive explicit expressions for the function $\phi$ for a variety of standard learning rates. 
The following result \emph{exactly} characterizes the trajectory of $\beta_k$, and is completely generic in that it makes no assumptions on the data, the learning rate schedule, or the initialization. A special case of this result with $\mu = 0$ and $\beta_0 = 0$ was also obtained in~\cite{raskutti2014early}.

\begin{proposition}[Trajectory] \label{prop:betak}
Let $X = U\Sigma_X V^T$ be any feature matrix and $y = X \beta_* + \varepsilon$. If $\beta_k$ are the parameters after $k$ steps of gradient descent for the ridge regression problem~\eqref{eq:ls-ridge} with regularization parameter $\mu \geq 0$, initialized at $\beta_0$ and with arbitrary learning rate schedule $\{ \eta_k\}_{k \geq 1}$, then we have that 
\begin{equation} \label{eq:betak_1}
    \beta_k = V\Phi(k; \mu)V^T\beta_0 + (I - V\Phi(k; \mu)V^T)(\mu n I + X^TX)^{\dagger}X^Ty.
\end{equation}
Moreover, recalling that $\tilde{\beta}_k = V^T\beta_k$, $\tilde{\beta}_* = V^T\beta_*$, and $\check{\varepsilon} = U^T \varepsilon$, we have
\begin{equation} \label{eq:betak_2}
    \tilde{\beta}_k - \tilde{\beta}_* = \Phi(k; \mu)(\tilde{\beta}_0 - \tilde{\beta}_*) + (\mu n I + \Lambda)^{\dagger}(I - \Phi(k; \mu))(\Sigma_X^T \check{\varepsilon} - \mu n \tilde{\beta}_*).
\end{equation}
\end{proposition}

The proof of Proposition~\ref{prop:betak} follows from matrix computations, and will be presented in Appendix~\ref{app:deferred_proofs}.

\begin{remark}[Why is early stopping beneficial?]
Proposition~\ref{prop:betak} suggests that for a wide variety of learning rate schedules $\{\eta_k\}_{k \geq 1}$, early stopping can be beneficial. Specifically, if the step size is sufficiently small and $\Phi(k;\mu)$ is monotonically decreasing, then the error from learning the signal $\Phi(k;\mu)(\tilde{\beta}_0 - \tilde{\beta}_*)$ is monotonically decreasing, while the second term $(\mu n I + \Lambda)^{\dagger}(I - \Phi(k; \mu))(\Sigma_X^T \check{\varepsilon} - \mu n \tilde{\beta}_*)$ corresponding to fitting the noise is monotonically increasing. This suggests that we might want to balance the two quantities by stopping early.
We make this intuition precise in Section~\ref{sec:should-we-stop} by providing conditions on the learning rates under which early stopping is beneficial (or not).
\end{remark}

\begin{remark}[Decoupling of the learning dynamics]
Equation~\eqref{eq:betak_1} shows that $\beta_k$ can be thought of as an affine combination of the initialization $\beta_0$ and the minimum norm solution $(\mu n I + X^TX)^\dagger X^T y$ with weights given by $V\Phi(k)V^T$.
Equation~\eqref{eq:betak_2} also shows that the dynamics of $\beta_k - \beta_*$ decouple when the parameters are expressed in the eigenbasis $V$: recalling that $\Phi(k;\mu)$ is a diagonal matrix, it can be rearranged to
\begin{equation} \label{eq:decouple}
    \tilde{\beta}_k - \tilde{\beta}_*
    = \Phi(k;\mu)\left[(\tilde{\beta}_0 - \tilde{\beta}_*) - (\mu n I + \Lambda)^{\dagger} (\Sigma_X^T \check{\varepsilon} - \mu n \tilde{\beta}_*) \right] + (\mu n I + \Lambda)^{\dagger} (\Sigma_X^T \check{\varepsilon} - \mu n \tilde{\beta}_*). 
\end{equation}
Thus, each coordinate of $\tilde{\beta}_k - \tilde{\beta}_*$ evolves independently of the others. 
\end{remark}

\subsection{Formulas for the function \texorpdfstring{$\phi$}{phi}} \label{sec:phi-exact-expressions}
For generic learning rate schedules, we may not have closed-form expressions for $\phi(k; \zeta, \{ \eta_i\} )$. However, for many common learning rate schedules, $\phi$ can indeed be evaluated in terms of analytic functions.
For \emph{constant learning rate schedules} (i.e., $\eta_k \equiv \eta$), we can see that Definition~\ref{defn:phi} is satisfied by
\[
    \phi(k; \zeta, \{\eta\}) = \left( 1 - \eta \zeta\right)^k \quad \text{and} \quad \phi(0;\zeta) = 1.
\]
The following two results give expressions for $\phi$ in terms of the Gamma function for \emph{learning rates with polynomial decay} or \emph{constant additive decay}.

\begin{proposition}[Polynomial decay] \label{prop:polynomial}
If $\eta_k = \eta/k^m$ for some integer $m \geq 1$, then we have that
\begin{equation} \label{eq:polynomial}
    \phi(k; \zeta) = \frac{1}{\Gamma(k+1)^m} \prod_{j=1}^m \Gamma(k+1 - \omega_j (\eta \zeta)^{1/m}) \quad \text{and} \quad \phi(0; \zeta) = \prod_{j=1}^m \Gamma(1 - \omega_j(\eta \zeta)^{1/m}),
\end{equation}
where $\omega_1, \ldots, \omega_m$ are the $m$th roots of unity, and $\Gamma(z) = \int^\infty_0 t^{z-1} e^{-t} \,\mathrm{d}t$ is the Gamma function.
\end{proposition}

\begin{proof}[Proof.]
First, note that using the roots of unity, we have the following polynomial factorization:
\[
    1 - \frac{\eta}{k^m} \zeta = \prod_{j=1}^m \left(1 - \frac{\omega_j}{k}(\eta \zeta)^{1/m}\right).
\]
Thus, given a learning rate schedule with polynomial decay with $\eta_k = \eta / k^m$, by using this identity and switching the order of the products, $\phi(k; \zeta)$ and $\phi(0; \zeta)$ must satisfy
\begin{align*}
    \phi(k; \zeta)
    &= \phi(0; \zeta) \cdot \prod_{i=1}^k \left( 1 - \frac{\eta}{i^m} \zeta \right) \\
    &= \phi(0; \zeta) \cdot \prod_{i=1}^k \prod_{j=1}^m \left(1 - \frac{\omega_j}{i}(\eta \zeta)^{1/m}\right)
    = \frac{\phi(0; \zeta)}{\Gamma(k+1)^m} \cdot \prod_{j=1}^m \prod_{i=1}^k \left(i - \omega_j(\eta \zeta)^{1/m}\right).
\end{align*}
From the fundamental property $\Gamma(z+1) = z \Gamma(z)$ of the Gamma function, we have
\[
    \prod_{i=1}^k \left(i - \omega_j(\eta \zeta)^{1/m}\right) = \frac{\Gamma(k+1 - \omega_j (\eta \zeta)^{1/m})}{\Gamma(1 - \omega_j (\eta \zeta)^{1/m})}.
\]
Thus, defining $\phi(k;\zeta)$ and $\phi(0;\zeta)$ as in~\eqref{eq:polynomial} leads to a valid expression according to Definition~\ref{defn:phi}.
\end{proof}

\begin{proposition}[Constant additive decay] \label{prop:constant-decay}
If $\eta_k = \eta_0 - k \eta$, then we have that 
\begin{equation} \label{eq:step}
    \phi(k; \zeta) = (\eta \zeta)^k \cdot \Gamma\left(k+1 + \frac{1-\eta_0 \zeta}{\eta \zeta}\right) \quad \text{and} \quad \phi(0; \zeta) =  \Gamma\left(1 + \frac{1-\eta_0 \zeta}{\eta \zeta}\right),
\end{equation}
where $\Gamma(z) = \int^\infty_0 t^{z-1} e^{-t} \,\mathrm{d}t$ is the Gamma function.
\end{proposition}
\begin{proof}[Proof.]
The function $\phi(k; \zeta)$ and $\phi(0; \zeta)$ must satisfy
\[
    \phi(k;\zeta)
    = \phi(0;\zeta) \cdot \prod_{i=1}^k (1 - \eta_0 \zeta + i \eta \zeta)
    = \phi(0;\zeta) \cdot (\eta \zeta)^k \prod_{i=1}^k \left(i + \frac{1 - \eta_0 \zeta}{\eta \zeta}\right).
\]
Using the property $\Gamma(z + 1) = z \Gamma (z)$ again verifies the validity of the claimed expressions.
\end{proof}

Finally, for \emph{learning rates with exponential decay}, $\phi$ can be evaluated in terms of a function known as the \emph{$q$-Pochhammer symbol}, which is defined by
\begin{equation} \label{eq:q-pochhammer}
    (a; q)_n = \prod_{i=0}^{n-1} (1 - aq^i).
\end{equation}
From the definition of the $q$-Pochhammer symbol~\eqref{eq:q-pochhammer}, we immediately obtain the following:

\begin{proposition}[Exponential decay] \label{prop:exponential}
If $\eta_k = \eta^k$, then we have that
\begin{equation}
    \phi(k; \zeta) = (\zeta; \eta)_{k+1} \quad \text{and} \quad \phi(0; \zeta) = 1 - \zeta.
\end{equation}
\end{proposition}

We can also obtain a formula for composite learning schedules by multiplying the corresponding $\phi$ functions. For example, suppose that we are given a cyclic learning rate schedule with period $T$ such that $\eta_{k+T} = \eta_k$ for all $k$. Then, for any iteration $k = qT + r$, we have that
\begin{equation}
    \phi(k;\zeta) = \left(\frac{\phi(T;\zeta)}{\phi(0;\zeta)} \right)^q \phi(r;\zeta).
\end{equation}

In summary, we can obtain closed-form expressions for $\phi$ for many (but not all) learning rate schedules used in practice.
One prominent learning rate schedule that we do not have a formula for is \emph{cosine annealing}~\cite{loshchilov2017sgdr}. In its standard form, cosine annealing chooses the learning rate at iteration $k$ to vary smoothly between a maximum value $\eta_{\mathrm{max}}$ and a minimum value $\eta_{\mathrm{min}}$ according to a cosine profile, typically of the form
\begin{equation} \label{eq:cosine_annealing}
    \eta_k = \eta_{\mathrm{min}} + \frac{1}{2}(\eta_{\mathrm{max}} - \eta_{\mathrm{min}}) \left( 1 + \cos\!\left(\frac{\pi k}{T}\right) \right),
\end{equation}
where $T$ is the annealing horizon. This schedule is widely used because it decreases the learning rate smoothly rather than through abrupt drops, and it is often combined with warm restarts that periodically reset the learning rate to a larger value.

\begin{remark}[Role of the closed-form filters]
The closed-form expressions for $\phi$ are useful because the effect of the learning rate schedule $\{ \eta_i \}_{i \geq 1}$ only enters the dynamics through
\[
    \Phi_j(k;\mu,\Lambda)
    = \frac{\phi(k;\mu+n^{-1}\Lambda_j)}
    {\phi(0;\mu+n^{-1}\Lambda_j)}
    = \prod_{i=1}^k
    \left(1-\eta_i(\mu+n^{-1}\Lambda_j)\right).
\]
Thus, once $\phi$ is known, we obtain explicit formulas for the parameters after $k$ iterations in Proposition~\ref{prop:betak}, which will be crucial for studying early stopping.
These formulas also make it possible to compare learning rate schedules through the rate at which they shrink different eigendirections. Hence, besides being algebraic conveniences, the closed forms also provide a way to analyze the regularization effect of different schedules.
\end{remark}

\section{Early stopping and generalized ridge regularization}

In this section, we show an equivalence between early stopping for the least squares problem with $\mu = 0$ and generalized ridge regularization for generic data and learning rate schedules.
Specifically, given a matrix $D \in \mathbb{R}^{p \times p}$, the \emph{generalized ridge regression problem} is the following:
\begin{equation} \label{eq:gen-ridge}
    \argmin_{\beta \in \mathbb{R}^p} \frac{1}{2n}\|y - X \beta\|^2 + \frac{1}{2}\| D \beta \|^2.
\end{equation}
By considering generalized ridge regression, we can allow for a different regularization strength for each coordinate of $\tilde{\beta}_k$ in the eigenspaces of $X^T X$.
Recall that $\Phi(T) \equiv \Phi(T; 0, \Lambda)$ is the diagonal matrix defined in~\eqref{eq:defn_Phi_matrix2}.
Then, using Proposition~\ref{prop:betak}, we can show that the early stopped solution after $T$ iterations solves a generalized ridge regression problem corresponding to some matrix $D$ depending on $\Phi(T)$:

\begin{theorem}[Early stopping $\implies$ generalized ridge regression] \label{thm:equivalence1}
Let $X$ be any feature matrix with rank $r$, and $y = X \beta_* + \varepsilon$. Suppose that $\beta_T$ are the parameters after $T$ steps of gradient descent for the least squares problem~\eqref{eq:ls-ridge} with $\mu = 0$, initialized at $\beta_0 = 0$, and with arbitrary learning rate schedule $\{ \eta_k\}_{k \geq 1}$. Assume that $\Phi_j(T) \ne 1$ for $j = 1, \ldots, r$. Then, the early stopped solution $\beta_T$ is the minimum norm solution to the generalized ridge regression problem~\eqref{eq:gen-ridge} with
\[
    D = \left(\frac{1}{n}\Lambda \Phi(T)(I-\Phi(T))^{\dagger}\right)^{1/2}V^T.
\]
\end{theorem}

\begin{remark}
Crucially, note that Theorem~\ref{thm:equivalence1} makes no assumptions on the data or the learning rate schedule besides the technical condition that $\Phi_j(T) \ne 1$, which is generically satisfied.
In fact, if we consider large learning rates, then $\Phi(T)$ could have negative entries! In this case the entries of $D$ would be complex. Prior work~\cite{yilmaz2022regularization} has shown that the optimal ridge regularization parameter can be negative. Here, we see that this is analogous to training with $\mu = 0$ and large learning rates.
A version of Theorem~\ref{thm:equivalence1} was proved in \cite[Lemma 3]{ali2019continuous} for constant step size schedules.
\end{remark}

The main idea behind the proof of Theorem~\ref{thm:equivalence1} is that we have a closed form expression for the minimum norm solution of the generalized ridge regression problem, which can be shown to coincide with the expression for $\beta_T$ from Proposition~\ref{prop:betak}.

\begin{proof}[Proof of Theorem~\ref{thm:equivalence1}.]
First, observe that the generalized ridge regression problem can be reformulated as an augmented least squares problem:
\[
    \frac{1}{2n}\|y - X \beta\|^2 + \frac{1}{2}\| D \beta \|^2 = \frac{1}{2n}\left\|\begin{bmatrix} y \\ 0_p \end{bmatrix} - \begin{bmatrix} X \\ \sqrt{n}D \end{bmatrix} \beta \right\|_2^2.
\]
The solution $\beta^{(D)}$ to the generalized ridge regression problem~\eqref{eq:gen-ridge} is given by
\begin{align*}
   \beta^{(D)}
   &= \left( \begin{bmatrix} X^T & \sqrt{n} D^T \end{bmatrix} \begin{bmatrix} X \\ \sqrt{n} D \end{bmatrix}\right)^{\dagger}  \begin{bmatrix} X^T & \sqrt{n} D^T \end{bmatrix} \begin{bmatrix} y \\ 0_p \end{bmatrix} \\
   &= \left(X^TX + n D^T D \right)^{\dagger}  X^T y \\
   &= V\left(\Lambda + n V^T D^T D V \right)^{\dagger}  V^T X^T(X\beta_* + \varepsilon) \\
   &= V\left(\Lambda + n V^T D^T D V \right)^{\dagger}  \Lambda \tilde{\beta_*} + V\left(\Lambda + n V^T D^T D V \right)^{\dagger}  \Sigma_X^T \check{\varepsilon}.
\end{align*}
Let $\tilde{\beta}^{(D)} := V^T \beta^{(D)}$.
By multiplying both sides by $V^T$ and subtracting $\tilde{\beta}_*$ from both sides, we obtain
\begin{align*}
       \tilde{\beta}^{(D)} - \tilde{\beta}_*
       &= \left [\left(\Lambda + n V^T D^T D V \right)^{\dagger}  \Lambda - I \right] \tilde{\beta}_* + \left(\Lambda + n V^T D^T D V \right)^{\dagger} \Sigma_X^T \check{\varepsilon}.
\end{align*}
Next, substituting in the chosen value of $D$, we have that
\begin{align*}
    \Lambda + n V^T D^T D V
    &= \Lambda + \Lambda \Phi(T)(I-\Phi(T))^{\dagger} .
\end{align*}
To simplify this expression, we note that all of the matrices are diagonal matrices. To handle both the underparameterized and overparameterized cases simultaneously, let $r = \mathrm{rank}(X)$, and write $A_{1:r,1:r}$ to denote the $r \times r$ principal submatrix of a matrix $A$. Then, we have that 
\[
    \Lambda = \begin{bmatrix} \Lambda_{1:r,1:r} & 0 \\ 0 & 0 \end{bmatrix}
    \quad\text{and}\quad
    \Phi(T) = \begin{bmatrix}
        \Phi(T)_{1:r, 1:r} & 0 \\ 0 & I
    \end{bmatrix}.
\]
Thus, we see that 
\begin{align*}
    &\Lambda + n V^T D^T D V  = \Lambda + \Lambda \Phi(T)(I-\Phi(T))^{\dagger}  \\
    &\quad= \begin{bmatrix} \Lambda_{1:r,1:r} & 0 \\ 0 & 0 \end{bmatrix} \left(I + \begin{bmatrix}
        \Phi(T)_{1:r, 1:r} & 0 \\ 0 & I
    \end{bmatrix} \begin{bmatrix}
        I - \Phi(T)_{1:r, 1:r} & 0 \\ 0 & 0
    \end{bmatrix}^\dagger\right) \\
    &\quad= \begin{bmatrix} \Lambda_{1:r,1:r} & 0 \\ 0 & 0 \end{bmatrix} \begin{bmatrix}
        I + \Phi(T)_{1:r, 1:r} (I - \Phi(T)_{1:r, 1:r})^{-1} & 0 \\ 0 &  I
    \end{bmatrix}.
\end{align*}
Observe that for the principal $r \times r$ submatrix, we have
\begin{align*}
    I + \Phi(T)_{1:r, 1:r} (I - \Phi(T)_{1:r, 1:r})^{-1}
    &= \left[(I - \Phi(T)_{1:r, 1:r}) + \Phi(T)_{1:r, 1:r} \right](I - \Phi(T)_{1:r, 1:r})^{-1} \\
    &= (I - \Phi(T)_{1:r, 1:r})^{-1}.
\end{align*}
Hence, we have that
\begin{align*}
    \left(\Lambda + n V^T D^T D V \right)
    &= \begin{bmatrix}
        \Lambda_{1:r,1:r} (I - \Phi(T)_{1:r, 1:r})^{-1} & 0 \\ 0 & 0
    \end{bmatrix}.
\end{align*}
Since $\Lambda_{1:r,1:r}$ is invertible and $I-\Phi(T)_{1:r,1:r}$ is invertible, the non-zero principal block is invertible. Therefore, taking the Moore--Penrose pseudoinverse amounts to inverting this principal block and leaving the null block equal to zero. Thus,
\[ 
    \left(\Lambda + n V^T D^T D V \right)^{\dagger} = \begin{bmatrix} (I-\Phi(T)_{1:r,1:r})\Lambda_{1:r,1:r}^{-1} & 0 \\ 0 & 0 \end{bmatrix}. 
\] 
Equivalently, since the matrices are diagonal, this can be written as 
\[ 
    \left(\Lambda + n V^T D^T D V \right)^{\dagger} = \Lambda^\dagger (I-\Phi(T)), \] where \[ I-\Phi(T) = \begin{bmatrix} I-\Phi(T)_{1:r,1:r} & 0 \\ 0 & 0 \end{bmatrix}. 
\] 
Hence, using the fact that the diagonal matrices commute, we have
\[ 
    \left(\Lambda + n V^T D^T D V \right)^{\dagger}\Lambda = \begin{bmatrix} I-\Phi(T)_{1:r,1:r} & 0 \\ 0 & 0 \end{bmatrix}. 
\] 
Therefore, 
\[ 
    \left[ \left(\Lambda + n V^T D^T D V \right)^{\dagger}\Lambda - I \right]\tilde{\beta}_* = \begin{bmatrix} -\Phi(T)_{1:r,1:r} & 0 \\ 0 & -I \end{bmatrix} \tilde{\beta}_* = \begin{bmatrix} \Phi(T)_{1:r,1:r} & 0 \\ 0 & I \end{bmatrix} (-\tilde{\beta}_*).
\]
Substituting this back in, we obtain
\[
    \tilde{\beta}^{(D)} - \tilde{\beta}_* = \begin{bmatrix} \Phi(T)_{1:r,1:r} & 0 \\ 0 & I \end{bmatrix} (-\tilde{\beta}_*) + \Lambda^\dagger (I-\Phi(T)) \Sigma_X^T \check{\varepsilon}. 
\]
From Proposition~\ref{prop:betak}, we see that with $\mu = 0$ and $\beta_0 = 0$, the parameter estimate after $T$ iterations also satisfies
\[
    \tilde{\beta}_T - \tilde{\beta}_* = \begin{bmatrix} \Phi(T)_{1:r, 1:r} & 0 \\ 0 & I \end{bmatrix}(-\tilde{\beta}_*) + \Lambda^{\dagger} (I - \Phi(T)_{1:r, 1:r})\Sigma_X^T\check{\varepsilon},
\]
which completes the proof.
\end{proof}

Next, we shall present a partial converse to Theorem~\ref{thm:equivalence1}. We show that for any regularization parameter $\mu$, the minimum norm solution of the ridge regression problem~\eqref{eq:ls-ridge} can be obtained via early stopping. However, similar to Theorem~\ref{thm:equivalence1} where we had to regularize each component of $\tilde{\beta}_k$ in the eigenspaces of $X^T X$ independently, we require a different (constant) learning rate for each direction.

\begin{theorem}[Ridge regularization $\implies$ early stopping]\label{thm:equivalence2}
Let $X$ be any feature matrix and $y = X \beta_* + \varepsilon$. For any regularization parameter $\mu$, let $\beta^{(\mu)} = \left(\mu n I + X^TX\right)^{\dagger} X^T y$ be the minimum norm solution to the kernel ridge regression problem~\eqref{eq:ls-ridge}. Suppose that for each $j$, we choose
\[
    \eta^{(j)} = \frac{n}{\mu n + \Lambda_j}
\]
to be the learning rate for the $j$th coordinate of $\tilde{\beta}$ (if $\mu n + \Lambda_j \;=\; 0$, then we choose $\eta^{(j)} = 0$). Then, after one step of gradient descent for the unpenalized least squares problem~\eqref{eq:ls-ridge} with $\mu = 0$, initialized at $\beta_0 = 0$, we obtain $\beta^{(\mu)}$.
\end{theorem}

\begin{proof}[Proof.]
Recall that the dynamics of $\beta_k - \beta_*$ decouple when expressed in the eigenbasis $V$ from~\eqref{eq:decouple}.
With $\beta_0 = 0$, one step of gradient descent for the unregularized least squares problem with step size $\eta^{(j)}$ for each coordinate of $\tilde{\beta}$ can be written as
\begin{align*}
    \tilde{\beta}_1
    &= \frac{1}{n} \begin{bmatrix}
        \eta^{(1)} & 0 & \ldots & & 0 \\
        0 & \eta^{(2)} & \ldots & & 0 \\
        \vdots & \vdots & \ddots & & \vdots \\
        & & & \eta^{(p-1)} & 0 \\
        0 & 0 & \ldots & 0 & \eta^{(p)} 
    \end{bmatrix} \left[\Lambda \tilde{\beta}_* + \Sigma_X^T \check{\varepsilon}\right] \\
    &= \left(\mu n I + \Lambda\right)^{\dagger}  \left[\Lambda \tilde{\beta}_* + \Sigma_X^T \check{\varepsilon}\right]\\
    &= V^T\left(\mu n I + V\Lambda V^T\right)^{\dagger}  \left[V\Lambda V^T \beta_* + X^T\varepsilon\right]\\
    &= V^T\left(\mu n I + X^TX\right)^{\dagger}  X^T\left[X\beta_* +\varepsilon\right] \\
    &= V^T\left(\mu n I + X^TX\right)^{\dagger} X^Ty.
\end{align*}
Multiplying both sides by $V$ shows that $\beta_1 = \left(\mu n I + X^TX\right)^{\dagger} X^Ty = \beta^{(\mu)}$, as desired.
\end{proof}

\begin{remark}[Coordinate-dependent stopping times]
A version of Theorem~\ref{thm:equivalence2} can be presented given any learning rate schedule $\{ \eta_i \}_{i \geq 1}$ such that $\phi(k; \zeta, \{ \eta_i \}) \to 0$ monotonically as $k \to \infty$ for any $\zeta$. In this case, for each coordinate $j$ of $\tilde{\beta}$, we would choose a separate stopping time $T_j$ such that
\[
    \Phi_j(T_j; \Lambda_j/n,\{\eta_i\})
\]
is closest to $\mu n/(\mu n + \Lambda_j)$. However, since we cannot guarantee equality for generic learning rate schedules $\{\eta_i\}_{i\geq 1}$, we can only obtain an approximation in this way.

Using different learning rates for different parameters is common in modern optimization; for example, differential learning rate methods assign separate learning rates to different parameter groups and have recently been studied from a Hessian-informed perspective~\cite{xu2025hessian}.
However, note that the coordinates of $\tilde{\beta}$ are with respect to the eigendirections of the empirical covariance $X^T X$. Since the gradient descent dynamics of the parameters decouple in this spectral basis, each coordinate has its own scalar shrinkage factor. Hence, the coordinate-dependent stopping times $T_j$ allow for an explicit comparison with ridge regularization, and should not be interpreted as a practical training prescription.
\end{remark}

\section{Should we stop early?} \label{sec:should-we-stop}

In the previous section, we showed that early stopping acts like a regularizer. However, that does not tell us whether (a) regularization is beneficial; and (b) if it is beneficial, what the optimal stopping time is. In this section, we provide conditions for when early stopping is beneficial (or not).

\subsection{Early stopped risk}
To determine when to stop, we need to understand the dynamics of the \emph{expected excess risk}
\begin{equation} \label{eq:expected_risk}
    R(\beta_k) := \mathbb{E}_\varepsilon[\mathcal{R}(\beta_k) \mid X] = \mathbb{E}_\varepsilon[\|\beta_k - \beta_*\|^2_\Sigma \mid X]
\end{equation}
during training, where the expectation is taken over the residual $\varepsilon = y - X \beta_*$ (which has i.i.d.\ coordinates $\varepsilon_i = y_i - x_i^T \beta_*$), conditional on the feature matrix $X$.
To formulate our result, we shall impose the following assumption on the first and second moments of the residual $\varepsilon$:
\begin{assumption} \label{assumption:noise}
    The coordinates of $\varepsilon$ satisfy $\mathbb{E}[\varepsilon_i \mid x_i] = 0$ and $\mathbb{E}[\varepsilon_i^2 \mid x_i] = \tau^2 < \infty$ for all $1 \leq i \leq n$.
\end{assumption}
Note that Assumption~\ref{assumption:noise} is more general than the common assumption that the conditional distribution of the residual $\varepsilon$ given $x$ is subgaussian~\cite{bartlett2020benign, xu2023towards, zou2023benign, raskutti2014early}. 

Recall that $\Phi(k; \mu)$ is the diagonal matrix defined in~\eqref{eq:defn_Phi_matrix2}, and $\tilde{\beta}_k = V^T\beta_k$ and $\tilde{\beta}_* = V^T\beta_*$ from~\eqref{eq:in-eigbasis}. Then, we can leverage the formula for $\tilde{\beta}_k - \tilde{\beta}_*$ given in Proposition~\ref{prop:betak} to obtain the following exact formula for the expected excess risk $R(\beta_k)$:

\begin{proposition}[Risk with ridge regularization] \label{prop:gen-risk-betak-mu}
Let $X = U \Sigma_X V^T$ be any feature matrix and $y = X \beta_* + \varepsilon$. Assume that Assumption~\ref{assumption:noise} holds. If $\beta_k$ are the parameters after $k$ steps of gradient descent for the ridge regression problem~\eqref{eq:ls-ridge} with ridge regularization parameter $\mu \geq 0$, initialized at any $\beta_0$, and with arbitrary learning rate schedule $\{ \eta_k\}_{k \geq 1}$, then
\begin{align*}
    R(\beta_k) &= \left\| \Sigma^{1/2} V \left[ \Phi(k; \mu)(\tilde{\beta}_0 - \tilde{\beta}_*) - \mu n (\mu n I + \Lambda)^\dagger (I - \Phi(k; \mu)) \tilde{\beta}_* \right] \right\|_2^2 \\
    &\quad+ \tau^2 \left\| \Sigma^{1/2} V (\mu n I + \Lambda)^\dagger (I - \Phi(k; \mu)) \Lambda^{1/2} \right\|_F^2.
\end{align*}
\end{proposition}

The proof of Proposition~\ref{prop:gen-risk-betak-mu} will be provided in Appendix~\ref{app:deferred_proofs}.
As an immediate corollary of Proposition~\ref{prop:gen-risk-betak-mu}, we can read off the expected excess risk for the ridgeless case with $\mu = 0$.

\begin{corollary}[Ridgeless risk] \label{cor:gen-risk-betak}
Let $X = U \Sigma_X V^T$ be any feature matrix and $y = X \beta_* + \varepsilon$. Assume that Assumption~\ref{assumption:noise} holds. If $\beta_k$ are the parameters after $k$ steps of gradient descent for the least squares problem~\eqref{eq:ls-ridge} with $\mu = 0$, initialized at any $\beta_0$, and with arbitrary learning rate schedule $\{ \eta_k\}_{k \geq 1}$, then, recalling that $\Phi(k) \equiv \Phi(k; 0)$,
\[
    R(\beta_k) = \left\|\Sigma^{1/2}V\Phi(k)(\tilde{\beta}_0 - \tilde{\beta}_*)\right\|_2^2 + \tau^2\left\|\Sigma^{1/2}V(I - \Phi(k))\Sigma_X^\dagger \right\|_F^2.
\]
\end{corollary}

\begin{remark}[Early stopping as regularization] Corollary~\ref{cor:gen-risk-betak} offers another perspective on the regularization effects of early stopping to complement Theorem~\ref{thm:equivalence1}, which describes the regularizing effects of early stopping on the solution obtained from the viewpoint of training the model when initialized zero. On the other hand, Corollary~\ref{cor:gen-risk-betak} allows us to understand early stopping from the perspective of the generalization error with generic initialization. Indeed, prior work~\cite{DobribanWager2018, cheng2024dimension} has shown that for generic data, the expected excess risk for the solution $\beta^{(\mu)}$ to the ridge regression problem~\eqref{eq:ls-ridge} with regularization parameter $\mu$ is given by 
\[
    R(\beta^{(\mu)}) = \mu^2 \beta_*^T(\hat{\Sigma} + \mu I)^{-1}\Sigma (\hat{\Sigma} + \mu I)^{-1} \beta_* + \frac{\tau^2}{n}\Tr\left(\Sigma \hat{\Sigma} (\hat{\Sigma} + \mu I)^{-2}\right),
\]
recalling that $\hat{\Sigma} = n^{-1} X^T X$.
From Corollary~\ref{cor:gen-risk-betak}, we see that the early stopped risk after $k$ iterations of gradient descent for the unpenalized least squares problem (with $\mu = 0$) is given by 
\[
    R(\beta_k) = (\beta_0 - \beta_*)^T(V\Phi(k)V^T) \Sigma (V\Phi(k)V^T) (\beta_0 - \beta_*) + \frac{\tau^2}{n} \Tr\left(\Sigma \hat{\Sigma}^{\dagger} (I - V\Phi(k)V^T)^2 \right).
\]
We see that this is similar to the risk for $\beta^{(\mu)}$ under the correspondence $V\Phi(k)V^T \leftrightarrow \mu(\hat{\Sigma}+\mu I)^{-1}$. Indeed, note that if $I - V \Phi(k) V^T = I - \mu (\hat{\Sigma} + \mu I)^{-1} = \hat{\Sigma}(\hat{\Sigma} + \mu I)^{-1}$, then $\hat{\Sigma}^{\dagger} (I - V \Phi(k) V^T)^2 = \hat{\Sigma} (\hat{\Sigma} + \mu I)^{-2}$. Recall that the $j$th diagonal entry of $\Phi(k) \equiv \Phi(k; 0)$ is equal to $\prod_{i=1}^k \bigl( 1 - \eta_i n^{-1} \Lambda_j \bigr)$.
\end{remark}

\subsection{When is early stopping beneficial?}

Given Proposition~\ref{prop:gen-risk-betak-mu}, all we need to do to find the optimal iteration $k$ to minimize the expected excess risk is, in principle, to differentiate $R(\beta_k)$ with respect to $k$ and compute the critical points. However, the issue is that the right-hand side is only defined for discrete values of $k$. To get around this technical problem, we make the following assumption on $\phi(k; \zeta)$, which implies that $\Phi(k; \mu)$ can be extended to a differentiable function of $k$:

\begin{assumption} \label{assumption:differentiable-extension}
For all fixed $\zeta$ and learning rate schedules $\{\eta_i\}_{i \geq 1}$ such that for all $i$, $\eta_i \le \zeta^{-1}$, the function $k \mapsto \phi(k; \zeta, \{\eta_i\})$ can be extended to a monotonic differentiable function on $[1, \infty)$.
\end{assumption}

The differentiability of the extension is satisfied by many learning rate schedules. For example, recall that with constant step sizes $\eta_k \equiv \eta$, $\phi(k; \zeta) = (1 - \eta \zeta)^k$, which can clearly be extended to the differentiable function $x \mapsto (1 - \eta \zeta)^x$ on $\mathbb{R}$.
Propositions~\ref{prop:polynomial} and~\ref{prop:constant-decay} show that learning rate schedules with polynomial decay ($\eta_k = \eta/k^m$) and constant additive decay ($\eta_k = \eta_0 - k\eta$) also satisfy Assumption~\ref{assumption:differentiable-extension} since the Gamma function is known to be differentiable on $[1, \infty)$.
Finally, learning rates with exponential decay ($\eta_k = \eta^k$) also satisfy this assumption; this is implied by Proposition~\ref{prop:exponential} and the differentiability of a suitable extension of the $q$-Pochhammer symbol, the technical details of which are established in Appendix~\ref{app:q-pochh}.

Additionally, we shall also make the following statistical assumption on $\beta_0 - \beta_*$, which is similar to a common statistical \emph{spherical prior assumption} that $\beta_* \sim \mathcal{N}(0, \sigma^2 I)$ with zero initialization~\cite{DobribanWager2018, AdSaSo2020, ali2019continuous}:
\begin{assumption} \label{assumption:spherical}
    The entries of $\beta_0 - \beta_*$ are i.i.d.\ and have mean 0 and variance $\sigma^2$.
\end{assumption}
This means that the results address the performance of generic signals on average. Under Assumption~\ref{assumption:spherical}, we shall also take an additional expectation and analyze the \emph{Bayes excess risk}:
\begin{equation} \label{eq:bayes_risk}
    R_B(\beta_k) := \mathbb{E}_{\beta_*-\beta_0}[R(\beta_k)].
\end{equation}

To build intuition, we begin by presenting our results on when early stopping for the least squares problem~\eqref{eq:ls-ridge} with regularization parameter $\mu = 0$ is beneficial.

\begin{theorem}[Early stopping] \label{thm:earlystopping}
Let $X = U \Sigma_X V^T$ be any feature matrix with $\mathrm{rank}(X) = r$, and $y = X \beta_* + \varepsilon$. Recall that $\Lambda_1$ is the largest eigenvalue of $X^T X$. Suppose that Assumptions~\ref{assumption:noise}, \ref{assumption:differentiable-extension}, and~\ref{assumption:spherical} hold.
Let $\beta_k$ be the parameters after $k$ steps of gradient descent for the unpenalized least squares problem~\eqref{eq:ls-ridge} with $\mu = 0$, initialized at any $\beta_0$.
If the learning rate schedule $\{ \eta_k \}_{k \geq 1}$ is such that $\eta_k \le 1 / (n^{-1} \Lambda_1)$ for all $k$, and for all $j = 1, \ldots, r$ we have that 
\begin{equation} \label{eq:earlystopping_condition}
    \lim_{k \to \infty} \frac{\phi(k; n^{-1} \Lambda_j)}{\phi(0; n^{-1} \Lambda_j)} < \frac{\tau^2}{\Lambda_j \sigma^2 + \tau^2},
\end{equation}
then there is a finite $T$ such that for all $k \ge T$, $R_B(\beta_k) \ge R_B(\beta_T)$. That is, early stopping is beneficial.
Furthermore, if $P := V^T\Sigma V$, then for all $k$, 
\begin{equation} \label{eq:earlystopping_lower}
   R_B(\beta_k) \ge \sigma^2\left[ \sum_{j=1}^r P_{jj}\frac{\tau^2}{\Lambda_j\sigma^2 + \tau^2} + \sum_{j=r+1}^p P_{jj}\right].
\end{equation}
\end{theorem}

\begin{proof}[Proof.]
From Corollary~\ref{cor:gen-risk-betak}, noting that $\Lambda^\dagger = (\Sigma_X^T \Sigma_X)^{\dagger}$ is a rank $r$ diagonal matrix, the expected excess risk after $k$ iterations is given by
\begin{align*} 
    R(\beta_k) &= (\tilde{\beta}_0 - \tilde{\beta}_*)^T \Phi(k) P \Phi(k) (\tilde{\beta}_0 - \tilde{\beta}_*) + \tau^2\Tr\left( \Lambda^\dagger (I-\Phi(k))P(I-\Phi(k)) \right)\\
    &= \sum_{i,j=1}^p P_{ij} \left[\frac{\phi(k; n^{-1}\Lambda_i)\phi(k; n^{-1}\Lambda_j)}{\phi(0; n^{-1}\Lambda_i)\phi(0; n^{-1}\Lambda_j)}(\tilde{\beta}_0 - \tilde{\beta}_*)_i(\tilde{\beta}_0 - \tilde{\beta}_*)_j \right] \\
    &\quad+ \sum_{j=1}^r P_{jj} \cdot \tau^2\frac{1}{\Lambda_j}\left(1-\frac{\phi(k; n^{-1}\Lambda_j)}{\phi(0; n^{-1}\Lambda_j)}\right)^2.
\end{align*}
Taking the expectation over $\beta_0 - \beta_*$ under Assumption~\ref{assumption:spherical}, we see that the cross terms vanish, and the Bayes excess risk is given by
\begin{equation} \label{eq:earlystopping_1}
    R_B(\beta_k) = \sum_{j=1}^{p} P_{jj} \cdot \frac{\phi(k;n^{-1}\Lambda_j)^2}{\phi(0;n^{-1}\Lambda_j)^2}\sigma^2 
    + \sum_{j=1}^r P_{jj} \cdot \tau^2\frac{1}{\Lambda_j}\left(1-\frac{\phi(k; n^{-1}\Lambda_j)}{\phi(0; n^{-1}\Lambda_j)}\right)^2.
\end{equation}
Taking the derivative in $k$ and noting that $\phi(k; n^{-1}\Lambda_j)$ is a constant for $j > r$, we get
\begin{align}
    \partial_k R_B(\beta_k)
    &= 2\sum_{j=1}^r P_{jj} \cdot \frac{\partial_k \phi(k; n^{-1}\Lambda_j)}{\phi(0; n^{-1}\Lambda_j)} \left[\frac{\phi(k; n^{-1}\Lambda_j)}{\phi(0; n^{-1}\Lambda_j)}\sigma^2 - \tau^2 \frac{1}{\Lambda_j}\left(1-\frac{\phi(k; n^{-1}\Lambda_j)}{\phi(0;n^{-1}\Lambda_j)}\right) \right] \nonumber\\
    &= 2\sum_{j=1}^r \frac{P_{jj}}{\Lambda_j} \cdot \frac{\partial_k \phi(k; n^{-1}\Lambda_j)}{\phi(0; n^{-1}\Lambda_j)} \left[\frac{\phi(k;n^{-1}\Lambda_j)}{\phi(0; n^{-1}\Lambda_j)} \Big(\sigma^2 \Lambda_j + \tau^2\Big) - \tau^2  \right]. \label{eq:risk-derivative-mu0}
\end{align}
Note that since $\eta_k \le 1/(n^{-1} \Lambda_1)$ for all $k$, $\phi(k; n^{-1}\Lambda_j) / \phi(0; n^{-1} \Lambda_j)$ is a non-increasing function of $k$ by Assumption~\ref{assumption:differentiable-extension}, and hence $\partial_k \phi(k; n^{-1}\Lambda_j) / \phi(0; n^{-1} \Lambda_j) \leq 0$. Since $\Lambda_j \ge 0$ and $P_{jj} \geq 0$ (because $P$ is a positive semidefinite matrix), the sign of the derivative $\partial_k R_B(\beta_k)$ is determined by 
\[
    \frac{\phi(k;n^{-1}\Lambda_j)}{\phi(0; n^{-1}\Lambda_j)} \Big(\sigma^2 \Lambda_j + \tau^2\Big) - \tau^2.
\]
Therefore, by using the assumption~\eqref{eq:earlystopping_condition} that 
\[
    \lim_{k \to \infty} \frac{\phi(k;n^{-1}\Lambda_j)}{\phi(0; n^{-1}\Lambda_j)} < \frac{\tau^2}{\sigma^2\Lambda_j + \tau^2},
\]
we have that for sufficiently large $k$,
\[
    \partial_k R_B(\beta_k) \ge 2 \sum_{j=1}^{r} \frac{P_{jj}}{\Lambda_j} \cdot \frac{\partial_k \phi(k; n^{-1} \Lambda_j)}{\phi(0; n^{-1} \Lambda_j)} (\tau^2 - \tau^2) = 0.
\]
This shows that the derivative of the expected excess risk is eventually positive, from which we conclude that we should have stopped earlier to minimize the risk.

To obtain the lower bound on the Bayes excess risk, observe that we can minimize the expected excess risk independently in each eigendirection. Since $\phi$ is monotonic, we see that solving 
\[
   \frac{\phi(k_j;n^{-1}\Lambda_j)}{\phi(0; n^{-1}\Lambda_j)} = \frac{\tau^2}{\sigma^2\Lambda_j + \tau^2}
\]
for $k_j$ achieves a global minimum in the $j$th eigendirection for $j \le r$. By plugging this into the expression~\eqref{eq:earlystopping_1} for the Bayes excess risk for each $j \le r$, we obtain the lower bound as follows. For $j \le r$, the corresponding term in the summation is $P_{jj}$ times 
\begin{align*}
    \sigma^2 \left(\frac{\tau^2}{\sigma^2\Lambda_j + \tau^2}\right)^2 + \tau^2\frac{1}{\Lambda_j}\left(1-\frac{\tau^2}{\sigma^2\Lambda_j + \tau^2}\right)^2
    &= \sigma^2 \left(\frac{\tau^2}{\sigma^2\Lambda_j + \tau^2}\right)^2 + \tau^2\frac{1}{\Lambda_j}\left(\frac{\sigma^2\Lambda_j}{\sigma^2\Lambda_j + \tau^2}\right)^2 \\
    &= \frac{\sigma^2 \tau^2}{\sigma^2\Lambda_j + \tau^2}.
\end{align*}
For $j > r$, note that $\phi(k;n^{-1}\Lambda_j) / \phi(0; n^{-1}\Lambda_j) = 1$. This completes the proof.
\end{proof}

\begin{remark}
The only reason that Assumption~\ref{assumption:spherical} is needed is to address the cross terms in~\eqref{eq:earlystopping_1}. However, another condition that would lead the cross terms to vanish is if the matrix $P = V^T \Sigma V$ is diagonal. This is satisfied if the features are isotropic (i.e., $\Sigma = I$), or more generically, if $\Sigma$ and $\hat{\Sigma}$ are simultaneously diagonalizable by the same basis of eigenvectors $V$. The latter is similar to a common requirement in works studying the problem of covariate shift, where assumptions on the alignment between the eigenspaces of the covariance matrices of the training and test distributions are required~\cite{tripuraneni2021covariate}.  

If $P$ is diagonal, then a version of Theorem~\ref{thm:earlystopping} showing that early stopping is beneficial for minimizing the expected excess risk $R(\beta_k)$ holds under the assumption that for all $j \leq \mathrm{rank}(X)$,
\[
    \lim_{k \to \infty} \frac{\phi(k; n^{-1} \Lambda_j)}{\phi(0; n^{-1} \Lambda_j)} < \frac{\tau^2}{\Lambda_j (\tilde{\beta}_0 - \tilde{\beta}_*)_j^2 + \tau^2},
\]
and a lower bound on the expected excess risk is given by
\[
    R(\beta_k) \geq \sum_{j=1}^r P_{jj} \cdot \frac{\tau^2(\tilde{\beta}_0 - \tilde{\beta}_*)_j^2}{\Lambda_j(\tilde{\beta}_0 - \tilde{\beta}_*)_j^2 + \tau^2} + \sum_{j=r+1}^p P_{jj} \cdot (\tilde{\beta}_0 - \tilde{\beta}_*)_j^2.
\]
\end{remark}

Next, as a counterpart to Theorem~\ref{thm:earlystopping}, we shall provide sufficient conditions for when early stopping for the unpenalized least squares problem~\eqref{eq:ls-ridge} with regularization parameter $\mu = 0$ is \emph{not beneficial}.

\begin{theorem}[Early stopping converse] \label{thm:earlystoppingConverse}
Consider the same setup as Theorem~\ref{thm:earlystopping}. If the learning rate schedule $\{ \eta_k \}_{k \geq 1}$ is such that $\eta_k \le 1 / (n^{-1} \Lambda_1)$ for all $k$, and for all $j = 1, \ldots, r$ we have that 
\begin{equation} \label{eq:earlystoppingConverse_condition}
    \lim_{k \to \infty}\frac{\phi(k; n^{-1} \Lambda_j)}{\phi(0; n^{-1} \Lambda_j)} \ge \frac{\tau^2}{\Lambda_j \sigma^2 + \tau^2},
\end{equation}
then early stopping is not beneficial. 
\end{theorem}

\begin{proof}[Proof.]
The proof is the same as for Theorem~\ref{thm:earlystopping}, except that the inequality for $\partial_k R_B(\beta_k)$ is reversed, and we deduce that the derivative of the expected excess risk is always negative. Hence, early stopping is not beneficial since the Bayes excess risk can always be decreased by further iterations.
\end{proof}

\begin{remark}[Step size schedules and late generalization] \label{rem:grokking}
In summary, Theorems~\ref{thm:earlystopping} and~\ref{thm:earlystoppingConverse} show that the learning rate schedule $\{ \eta_k \}_{k \geq 1}$ affects whether early stopping is beneficial or not.
Theorem~\ref{thm:earlystopping} provides a sufficient condition for when early stopping is beneficial. In particular, if $\phi(k; \zeta) \to 0$ as $k \to \infty$ for all $\zeta$ so that the limit in the left-hand side of~\eqref{eq:earlystopping_condition} is zero, then we see that early stopping is \emph{always beneficial}, \emph{independent of the spectrum of the covariance matrices of the training and test data.} 
Examples of learning rate schedules that satisfy this assumption include constant learning rates ($\eta_k \equiv \eta < 1/(n^{-1} \Lambda_1)$), learning rates with linear decay ($\eta_k = \eta/k$), or more generally if the learning rates satisfy $\sum_{k = 1}^\infty \eta_k = \infty$.

We can interpret \emph{late generalization} or \emph{grokking} as the phenomenon where we want to keep training, even after overfitting the noise. On the other hand, when early stopping is beneficial, we have shown that we do not want to overfit the noise. Hence, we show that for many different learning rate schedules, linear models trained by gradient descent do not exhibit late generalization.

Theorem~\ref{thm:earlystoppingConverse} allows us to construct examples of learning rate schedules for which it is possible that early stopping \emph{is not beneficial}, namely, when the left-hand side of~\eqref{eq:earlystoppingConverse_condition} is non-zero and large enough, relative to $\tau^2$, $\sigma^2$, and the spectrum of $X^T X$.
Possible examples include fast decaying learning rates with polynomial decay ($\eta_k = \eta/k^m$ with $m > 1$) or exponential decay ($\eta_k = \eta^k$ with $\eta < 1$).
\end{remark}

Having built intuition in the ridgeless setting, we now characterize when early stopping is beneficial for solving the ridge regression problem~\eqref{eq:ls-ridge} with regularization parameter $\mu \geq 0$.

\begin{theorem}[Early stopping with ridge regularization] \label{thm:earlystopping-mu}
Let $X = U \Sigma_X V^T$ be any feature matrix with rank $r$, and $y = X \beta_* + \varepsilon$.
Recall that $\Lambda_1$ is the largest eigenvalue of $X^T X$.
Suppose that Assumptions~\ref{assumption:noise}, \ref{assumption:differentiable-extension}, and~\ref{assumption:spherical} hold.
Let $\beta_k$ be the parameters after $k$ steps of gradient descent for the ridge regression problem~\eqref{eq:ls-ridge} with regularization parameter $\mu \geq 0$, initialized at $\beta_0 = 0$.
If the learning rate schedule $\{ \eta_k \}_{k \geq 1}$ is such that $\eta_k \leq (\mu + n^{-1} \Lambda_1)^{-1}$ for all $k$, and for all $j = 1, \ldots, r$ we have that
\begin{equation} \label{eq:earlystopping-mu_condition1}
    \lim_{k \to \infty} \frac{\phi(k; \mu + n^{-1} \Lambda_j)}{\phi(0; \mu + n^{-1} \Lambda_j)} < \frac{\tau^2 - \sigma^2 \mu n}{\tau^2 + \sigma^2 \Lambda_j},
\end{equation}
then there exists a finite $T$ such that for all $k \geq T$, $R_B(\beta_k) \geq R_B(\beta_T)$. That is, early stopping is beneficial.

On the other hand, under the same assumptions, if for all $j = 1, \ldots, r$, we have that
\begin{equation} \label{eq:earlystopping-mu_condition2}
    \lim_{k \to \infty} \frac{\phi(k; \mu + n^{-1} \Lambda_j)}{\phi(0; \mu + n^{-1} \Lambda_j)} \geq  \frac{\tau^2 - \sigma^2 \mu n}{\tau^2 + \sigma^2 \Lambda_j},
\end{equation}
then early stopping is not beneficial.
\end{theorem}

The proof of Theorem~\ref{thm:earlystopping-mu} follows the same strategy as for the ridgeless case in Theorem~\ref{thm:earlystopping} but involves more technical calculations, and will be deferred to Appendix~\ref{app:deferred_proofs}.

\begin{figure}[!htb]
    \centering
    \includegraphics[width=0.49\linewidth]{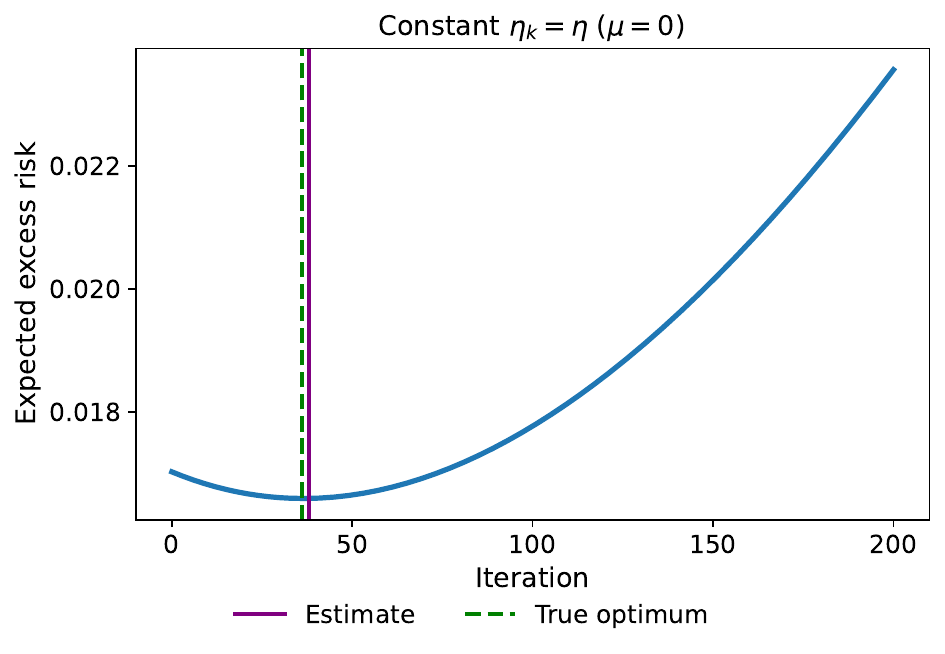}
    \includegraphics[width=0.49\linewidth]{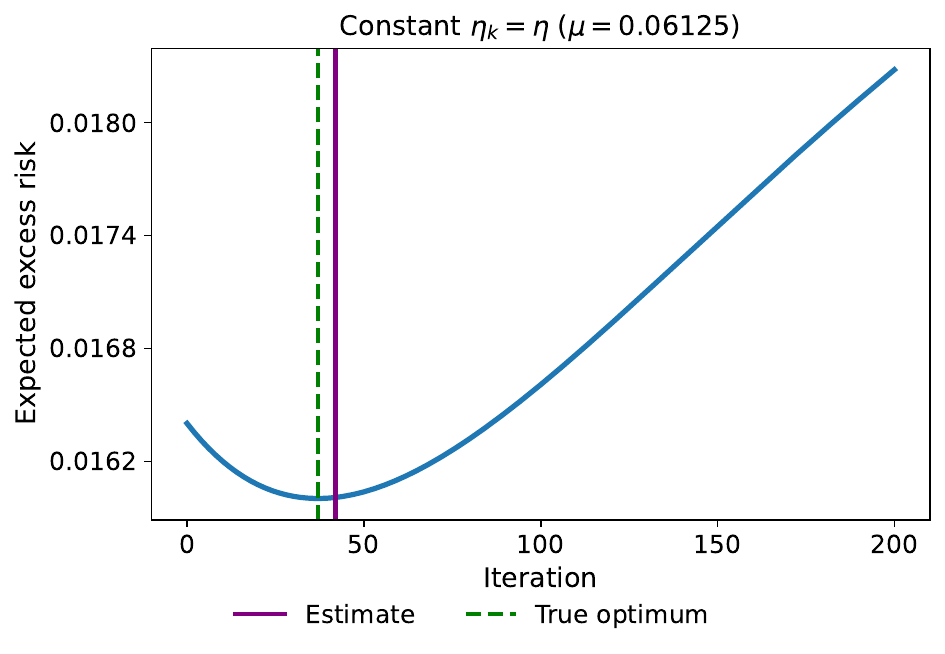}\\
    \includegraphics[width=0.49\linewidth]{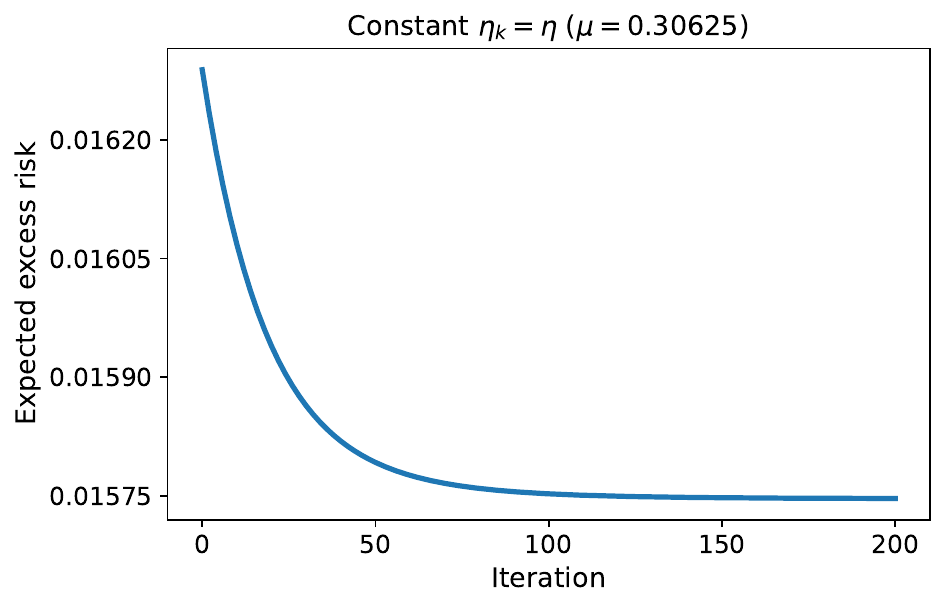}
    \includegraphics[width=0.49\linewidth]{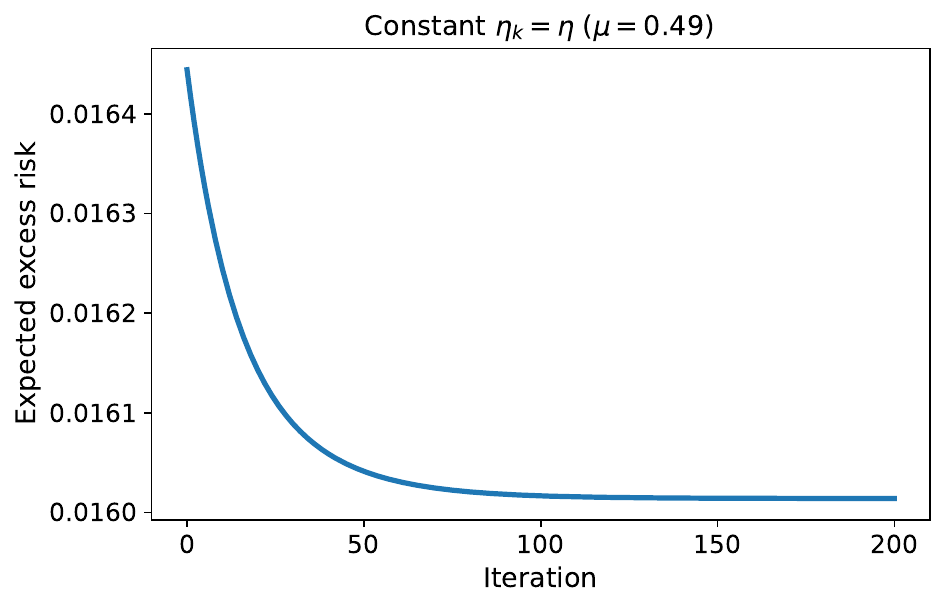}
    \caption{
    Bayes excess risk trajectories for gradient descent with constant learning rate $\eta = 0.01/(n^{-1}\Lambda_1)$ on four ridge regression problems with $n = 40$, $p = 100$, $\tau^2 = 0.35^2$, $\sigma^2 = 0.01$, $\mu^* = 0.30625$, and regularization parameter $\mu \in \{ 0, 0.06125, 0.30625, 0.49 \}$.
    When $\mu < \mu^*$ (top), early stopping is beneficial: the green dashed line shows the empirical optimal stopping time, and the purple line shows our estimated stopping time (see~\eqref{eq:when-to-stop-k} later).
    For $\mu \geq \mu^*$ (bottom), early stopping is not beneficial, which aligns with the predictions of the theory.
    }
    \label{fig:early-lam}
\end{figure}

\begin{remark}[Early stopping for ridge regression] \label{rem:grokking-ridge} Theorem~\ref{thm:earlystopping-mu} extends Theorems~\ref{thm:earlystopping} and~\ref{thm:earlystoppingConverse} by providing conditions under which early stopping remains beneficial (or not) even in the presence of explicit ridge regularization. This reveals an interesting implication. Specifically, if $\mu \ge \tau^2 / (n \sigma^2)$, then early stopping is never beneficial. In other words, if the explicit ridge regularization is too strong, then the benefits of early stopping diminish. 
This particular result was also established concurrently by Stark and Steinerberger~\cite[Theorem~2.5]{StarkSteinerberger2025} for gradient descent with constant learning rates.

Suppose that we train our model until convergence for the ridge regularized problem. Prior work~\cite{nakkiran2021optimal} has shown that $\mu^* = \tau^2 / (n \sigma^2)$ is an optimal ridge regularization parameter that minimizes the generalization error.
Hence, if we pick $\mu = \mu^*$, then early stopping is not beneficial! However, this does not imply that the expected excess risk cannot be improved by using \emph{both} ridge regularization and early stopping. In particular, the optimal early stopped solution with $\mu < \mu^*$ may outperform the converged solution with $\mu = \mu^*$.

We validate this numerically in Figure~\ref{fig:early-lam}. In this experiment, we sample $n = 40$ data points $x \sim z \Sigma^{1/2} \in \mathbb{R}^{p}$ with $p=100$, where $z \sim \mathcal{N}(0, p^{-1} I)$ and $\Sigma$ is a diagonal matrix with entries $\Sigma_{jj} = j^{-2}$. We sample $800$ parameters $\beta_* \sim \mathcal{N}(0, p^{-1} I)$ and independent mean-zero Gaussian noise with variance $\tau^2 = 0.35^2$; i.e., we work in the spherical prior setting with $\sigma^2 = 0.01$. Thus, the optimal ridge regularization parameter is $\mu^* = \tau^2 / (n \sigma^2) = 0.30625$.
We consider gradient descent with constant learning rate $\eta_k \equiv \eta = 0.01 / (n^{-1} \Lambda_1)$, where $\Lambda_1$ is the largest eigenvalue of $X^T X$, for four ridge regression problems with regularization parameters
\[
    \mu \in \{ 0, 0.5 \tau^2, 2.5 \tau^2, 4 \tau^2 \}
    = \{ 0, 0.06125, 0.30625, 0.49 \}.
\]
We see that for $\mu < \mu^*$, early stopping is beneficial, while for $\mu \geq \mu^*$, early stopping is not beneficial. Additionally, we see that for $\mu < \mu^*$, the early-stopped risk is similar to the converged risk obtained with the optimal ridge parameter $\mu = \mu^*$.
\end{remark}

\section{Optimal stopping time estimate} \label{sec:time-estimate}

Having shown that early stopping is beneficial for a wide variety of learning rate schedules, we shall now provide an estimate for the optimal stopping time.
Recall that we have the expression~\eqref{eq:risk-derivative-mu0} for the derivative of the Bayes excess risk in the ridgeless case from the proof of Theorem~\ref{thm:earlystopping}. More generally, in the proof of Theorem~\ref{thm:earlystopping-mu} (see~\eqref{eq:risk-derivative}), we show that the derivative of the Bayes excess risk with ridge regularization parameter $\mu \geq 0$ is given by
\[
    \partial_k R_B(\beta_k) = 2 \sum_{j=1}^{r} P_{jj} \cdot \partial_k \Phi_j(k) \cdot \frac{\Lambda_j}{(\mu n + \Lambda_j)^2} \left[\Phi_j(k) \cdot (\sigma^2 \Lambda_j + \tau^2) + (\sigma^2 \mu n - \tau^2) \right].
\]
Solving the above for a critical point is quite challenging. Therefore, similar to what was done in~\cite{AdSaSo2020}, we can determine the optimal stopping time for each eigendirection by setting each summand to zero individually.
That is, recalling that $\Phi_j(k) = \phi(k; \mu + n^{-1} \Lambda_j) / \phi(0; \mu + n^{-1} \Lambda_j)$, for each $j = 1, \ldots, \mathrm{rank}(r)$, we need to find $k_j$ such that
\begin{equation}
    \prod_{i=1}^k (1 - \eta_i (\mu + n^{-1} \Lambda_j))
    = \frac{\tau^2 - \sigma^2 \mu n}{\tau^2 + \sigma^2 \Lambda_j}.
\end{equation}
We may assume that $\mu \le \tau^2 / (n \sigma^2)$, otherwise early stopping is never beneficial (Remark~\ref{rem:grokking-ridge}). By taking the logarithm of both sides and using the first order expansion of $\log(1+x)$, we obtain
\[
    -\log\left(\frac{\tau^2 + \sigma^2 \Lambda_j}{\tau^2 - \sigma^2 \mu n}\right) = \sum_{i=1}^{k_j}\log(1 - \eta_i (\mu + n^{-1} \Lambda_{j})) \approx -(\mu + n^{-1} \Lambda_{j}) \cdot \sum_{i=1}^{k_j} \eta_i.
\]
Hence, the choice of $k_j$ should satisfy
\begin{equation} \label{eq:when-to-stop}
    \sum_{i=1}^{k_j} \eta_i \approx \frac{\log\left(\frac{\tau^2 + \sigma^2 \Lambda_j}{\tau^2 - \sigma^2 \mu n}\right)}{\mu + n^{-1} \Lambda_{j}}.
\end{equation}
Note that as $\mu$ approaches $\tau^2 / (n \sigma^2)$, the optimal stopping time goes to infinity. 

The expression~\eqref{eq:when-to-stop} provides a stopping time estimate for each eigendirection. A single estimate can be obtained by using the mean of the eigenvalues; i.e., replacing $\Lambda_j$ in~\eqref{eq:when-to-stop} by the mean $\bar{\Lambda} := r^{-1} \sum_{j=1}^r \Lambda_j$.
This should be understood as a representative-eigenvalue approximation. Since early stopping acts as a spectral regularizer, the optimal stopping time is naturally eigendirection-dependent. However, when the spectrum is sufficiently concentrated, one might expect the bulk behavior to be well captured by a typical eigenvalue.
This is analogous to the approximation in~\cite{AdSaSo2020}, where the per-mode stopping criterion is converted into a single stopping time by evaluating it at a modal eigenvalue.

Putting everything together, Theorem~\ref{thm:earlystopping-mu} and the heuristics discussed above lead to the following estimate of the optimal stopping time:
\begin{equation} \label{eq:when-to-stop-k}
    \hat{k} = \argmin\limits_k \left\{\sum_{i=1}^{k} \eta_i > \frac{\log\left(\frac{\tau^2 + \sigma^2 \bar{\Lambda}}{\tau^2 - \sigma^2 \mu n}\right)}{\mu + n^{-1} \bar{\Lambda}} \right\}.
\end{equation}
Here, $\bar{\Lambda} = \Tr(X^T X)/ \mathrm{rank}(X^T X)$, recalling that $\mu$ is the ridge regularization parameter, $\tau^2$ is the variance of each entry of the residual under Assumption~\ref{assumption:noise} (i.e., the strength of the noise), and $\sigma^2$ is the variance of $\beta_0 - \beta_*$ under the prior in Assumption~\ref{assumption:spherical} (i.e., the strength of the signal). We test the performance of this estimate empirically in the next section.

The closest comparison for our proposed stopping time is by Advani et al.~\cite{AdSaSo2020}. If $\mu = 0$ and constant step sizes $\eta_i \equiv \eta$ are used, then~\eqref{eq:when-to-stop} leads to the estimate
\[
    \hat{k}_j \approx \frac{\log\left(1 + \frac{\sigma^2}{\tau^2}\Lambda_{j} \right)}{\eta n^{-1}} \Lambda_{j}.
\]
This corresponds to the estimate~\cite[Eq.\ (16)]{AdSaSo2020} for the ridgeless, gradient flow case.
Our estimate~\eqref{eq:when-to-stop} extends this estimate to discrete gradient descent with arbitrary spectra, non-constant learning rate schedules through the cumulative budget $\sum_{i=1}^k \eta_i$, and explicit ridge regularization.

We emphasize that~\eqref{eq:when-to-stop-k} is a Bayes risk estimate. The exact optimal stopping time for any fixed realization of $\beta_*$ will depend on the alignment of $\beta_*$ with the eigenspaces of $X^T X$. Under Assumption~\ref{assumption:spherical}, this dependence is averaged out and replaced by the scalar signal strength $\sigma^2$. Thus, the estimate does not require knowledge of the underlying signal $\beta_*$, although it does require estimates of population quantities such as the signal-to-noise ratio $\sigma^2 / \tau^2$.

\subsection{Experimental validation}

We run experiments on real feature matrices from the MNIST~\cite{lecun1998gradient} and CIFAR10~\cite{krizhevsky2009learning} datasets, which are classic collections of handwritten digits and color images. For each dataset, we sample $n = 1,000$ images and flatten the image vectors to form the feature matrix $X \in \mathbb{R}^{n\times p}$. For MNIST, $p = 784$, while for CIFAR10, $p = 3,072$; this situates the experiments in the underparameterized and overparameterized regimes, respectively. We sample $\beta_* \sim \mathcal{N}(0, p^{-1} I_p)$, which is in the spherical prior setting with $\sigma^2= p^{-1}$, and generate responses
\[
    y = X \beta_* + \varepsilon,
    \qquad
    \text{where } \varepsilon = (\varepsilon_i)_{i=1}^n \text{ with } \varepsilon_i \sim \mathcal{N}(0,\tau^2) \text{ and } \tau = 1.
\]
We estimate the Bayes excess risk by averaging over $800$ independent draws of $\beta_*$ and $\varepsilon$.
For each dataset, we run gradient descent for $T = 500$ iterations with base learning rate $\eta = 0.9 / (n^{-1} \Lambda_1)$, where $\Lambda_1$ is the largest eigenvalue of $X^T X$.
We compare six learning rate schedules: 
\begin{enumerate}
    \item a constant schedule $\eta_k=\eta$;
    \item a polynomial schedule $\eta_k=\eta/k^{0.4}$;
    \item a linearly decaying schedule $\eta_k=\max(\eta-(k-1)\eta/300,0)$;
    \item an exponentially decaying schedule $\eta_k=\eta\cdot 0.9825^{k-1}$;
    \item a cosine annealing schedule~\eqref{eq:cosine_annealing} with minimum rate $0.05 \eta$, maximum rate $\eta$, and period $200$;
    \item and a random schedule with learning rates sampled uniformly from $[0,3\eta]$.
\end{enumerate}
We repeat this experiment with different ridge regularization parameters $\mu \in [0, 0.5 \mu^*, 1.5 \mu^*]$, where $\mu^* = \tau^2 / (n \sigma^2)$ is the optimal ridge regularization parameter.

\begin{figure}[!htb]
    \centering
    \includegraphics[width=\textwidth, trim={0 1.5cm 0 0}, clip]{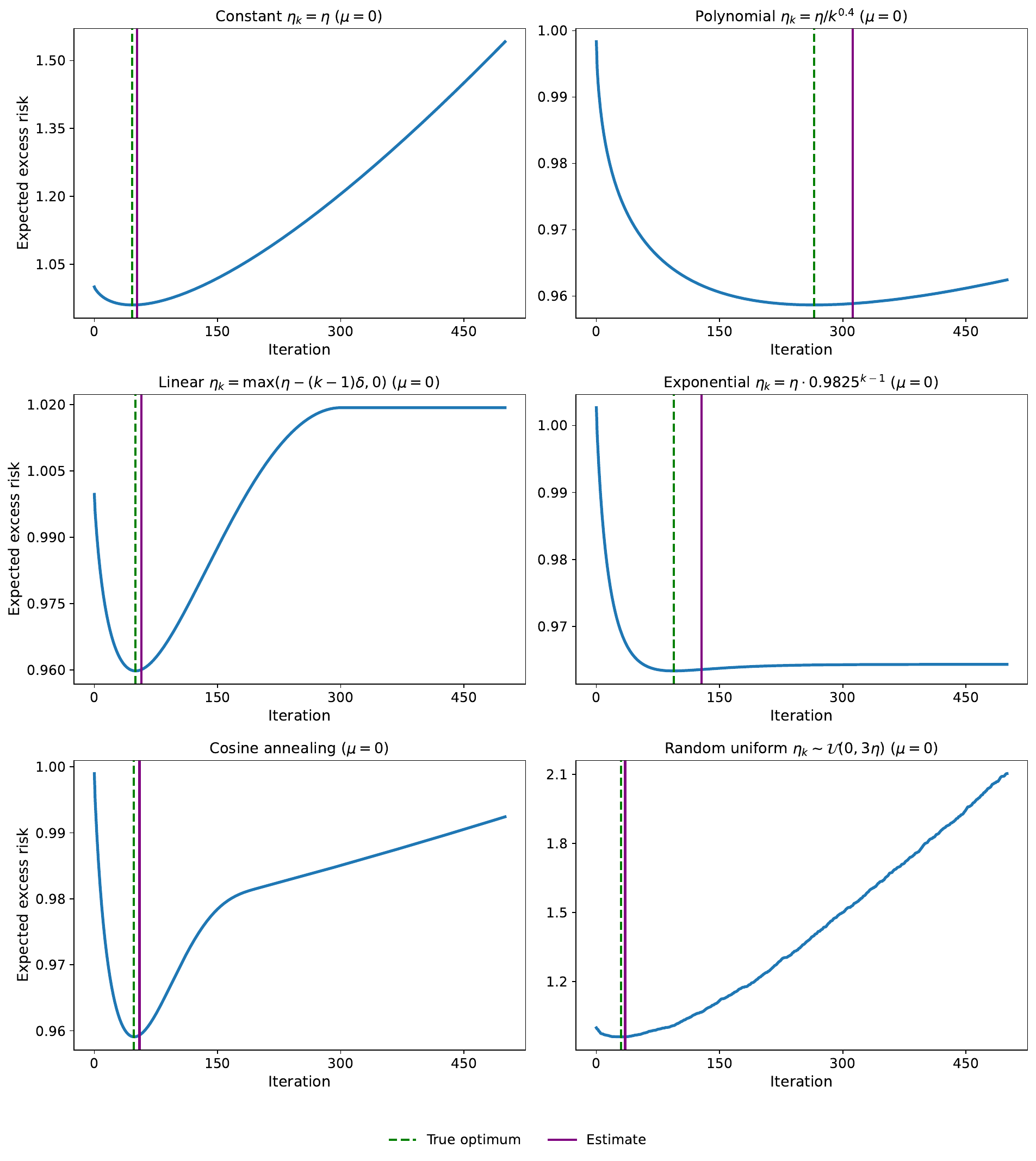}
    \caption{
    Bayes excess risk trajectories for the unregularized least squares problem ($\mu = 0$) with the MNIST dataset. The six panels compare different learning rate schedules. The green dashed line shows the empirical optimal stopping time, and the purple line shows our estimated stopping time~\eqref{eq:when-to-stop-k}.
    }
    \label{fig:real-mnist-mu0}
\end{figure}

\begin{figure}[!htb]
    \centering
    \includegraphics[width=\textwidth, trim={0 1.5cm 0 0}, clip]{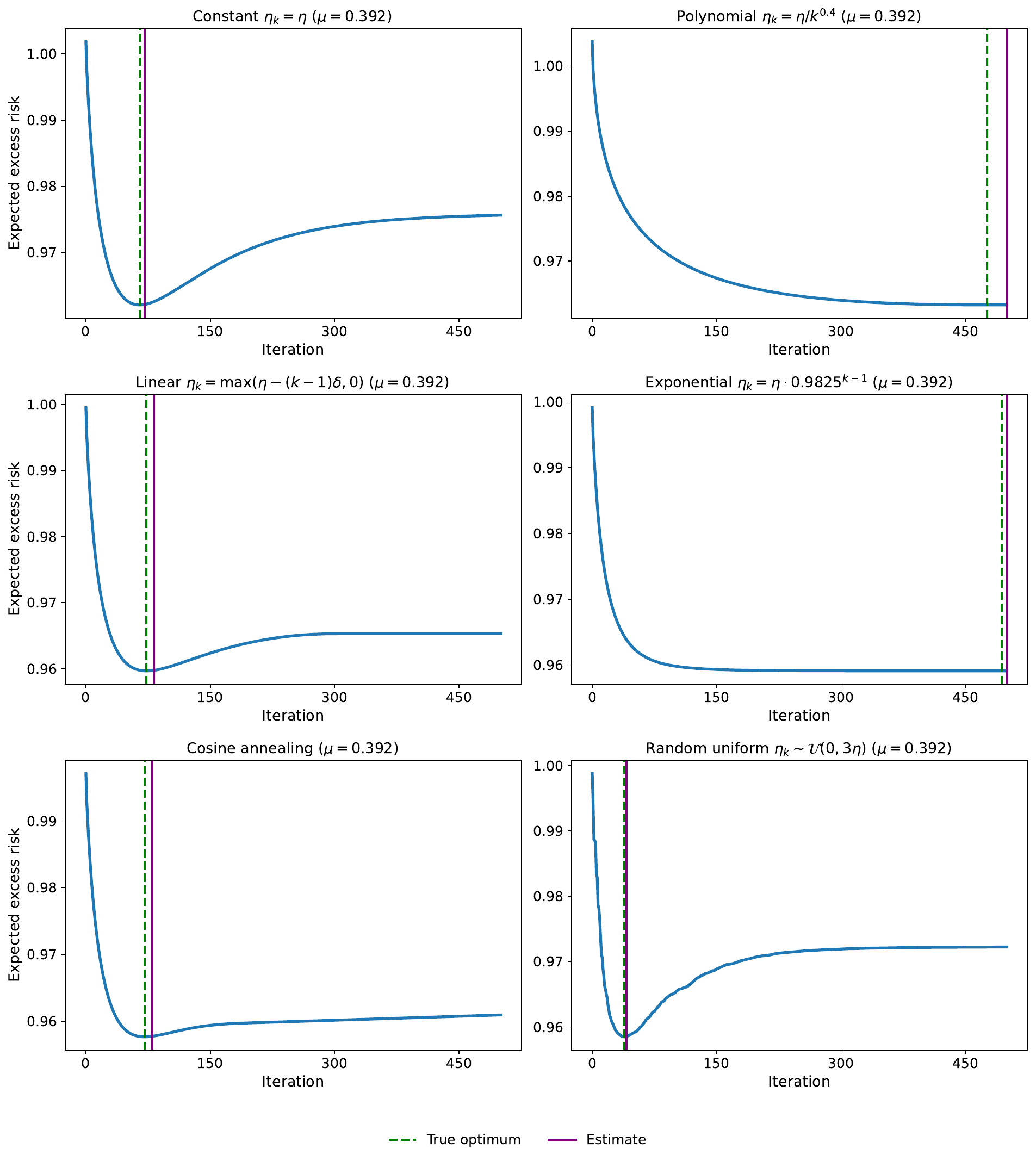}
    \caption{
    Bayes excess risk trajectories for the ridge-regularized least squares problem with $\mu = 0.392$ with the MNIST dataset. This corresponds to $\mu = \mu^* / 2$, where $\mu^* = 0.784$ is the optimal ridge parameter. The six panels compare different learning rate schedules. The green dashed line shows the empirical optimal stopping time, and the purple line shows our estimated stopping time~\eqref{eq:when-to-stop-k}.
    }
    \label{fig:real-mnist-mu-half}
\end{figure}

\begin{figure}[!htb]
    \centering
    \includegraphics[width=\textwidth]{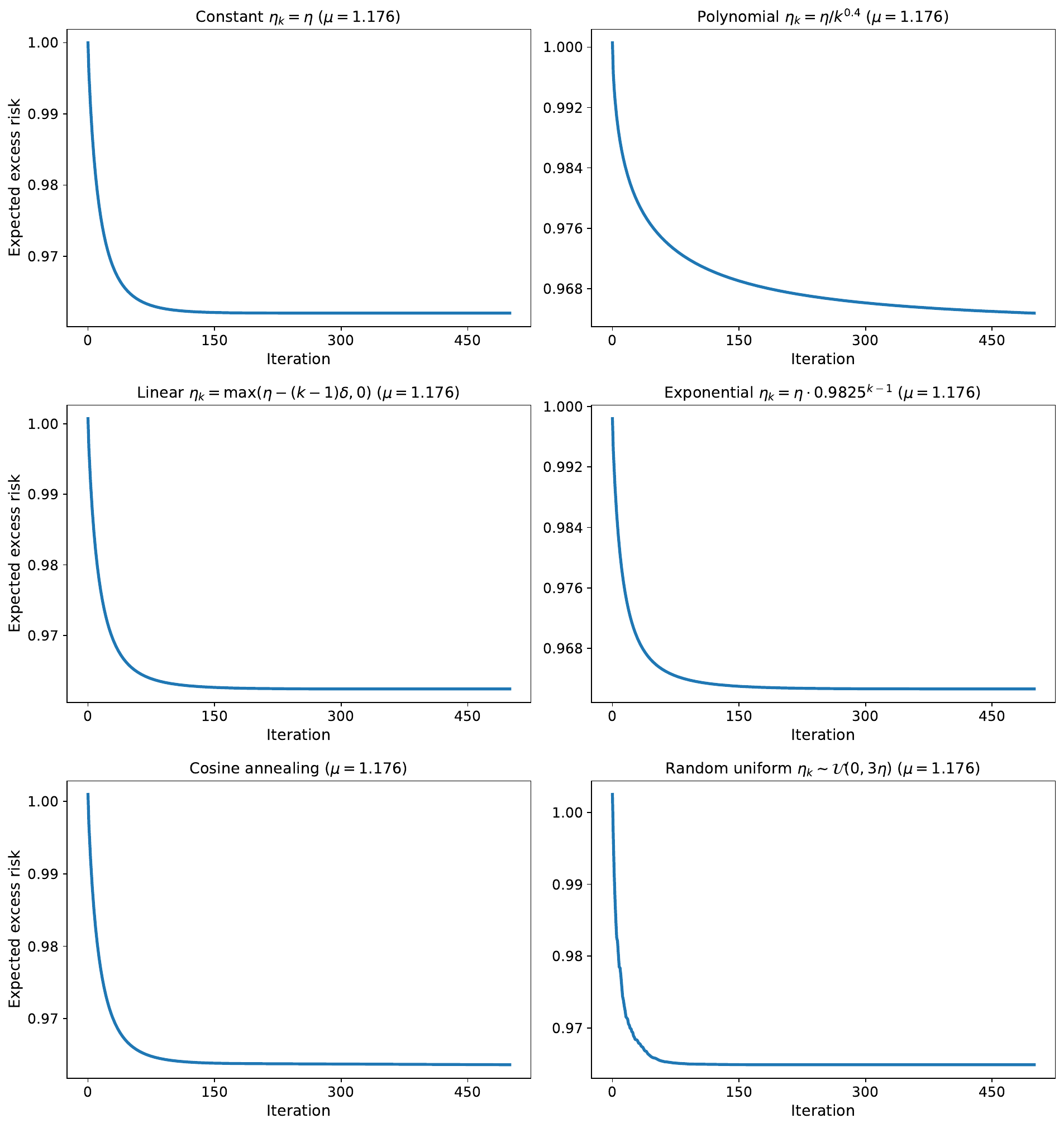}
    \caption{
    Bayes excess risk trajectories for the ridge-regularized least squares problem with $\mu = 1.176$ with the MNIST dataset. This corresponds to $\mu = 1.5\mu^*$, where $\mu^* = 0.784$ is the optimal ridge parameter. The six panels compare different learning rate schedules. Since $\mu > \mu^*$, the theory predicts that early stopping is not beneficial, and accordingly no lines for the optimal or estimated stopping time are plotted.
    }
    \label{fig:real-mnist-mu-large}
\end{figure}

The results for the MNIST dataset are consistent with the behavior predicted by our theory:
\begin{itemize}
    \item Figure~\ref{fig:real-mnist-mu0}: In the unregularized case ($\mu=0$), early stopping is beneficial across most learning rate schedules except for the exponentially decaying schedule, which decays too quickly. The estimated stopping times~\eqref{eq:when-to-stop-k} are consistently close to the empirical optimal stopping times. When there is a larger difference (e.g., for the polynomial or exponentially decaying schedules), the stopped risks are very similar. 
    \item Figure~\ref{fig:real-mnist-mu-half}: When $\mu > 0$ but $\mu < \mu^*$, the same qualitative behavior persists, although the amount of regularization already supplied by the ridge penalty reduces the improvement obtained from stopping early.
    \item Figure~\ref{fig:real-mnist-mu-large}: In contrast, when $\mu > \mu^*$ the risk curves are monotonically decreasing and training to convergence is optimal. This is in agreement with the prediction that early stopping is not beneficial once the explicit ridge regularization exceeds the optimal ridge parameter (Remark~\ref{rem:grokking}).
\end{itemize}
These experiments also illustrate that the iteration at which the optimum occurs can vary substantially across learning rate schedules.
The same experiments were run using the CIFAR10 dataset, and we can observe the same qualitative results in Figures~\ref{fig:real-cifar-mu0}--\ref{fig:real-cifar-mu-large}.

\begin{figure}[!htb]
    \centering
    \includegraphics[width=\textwidth, trim={0 1.5cm 0 0}, clip]{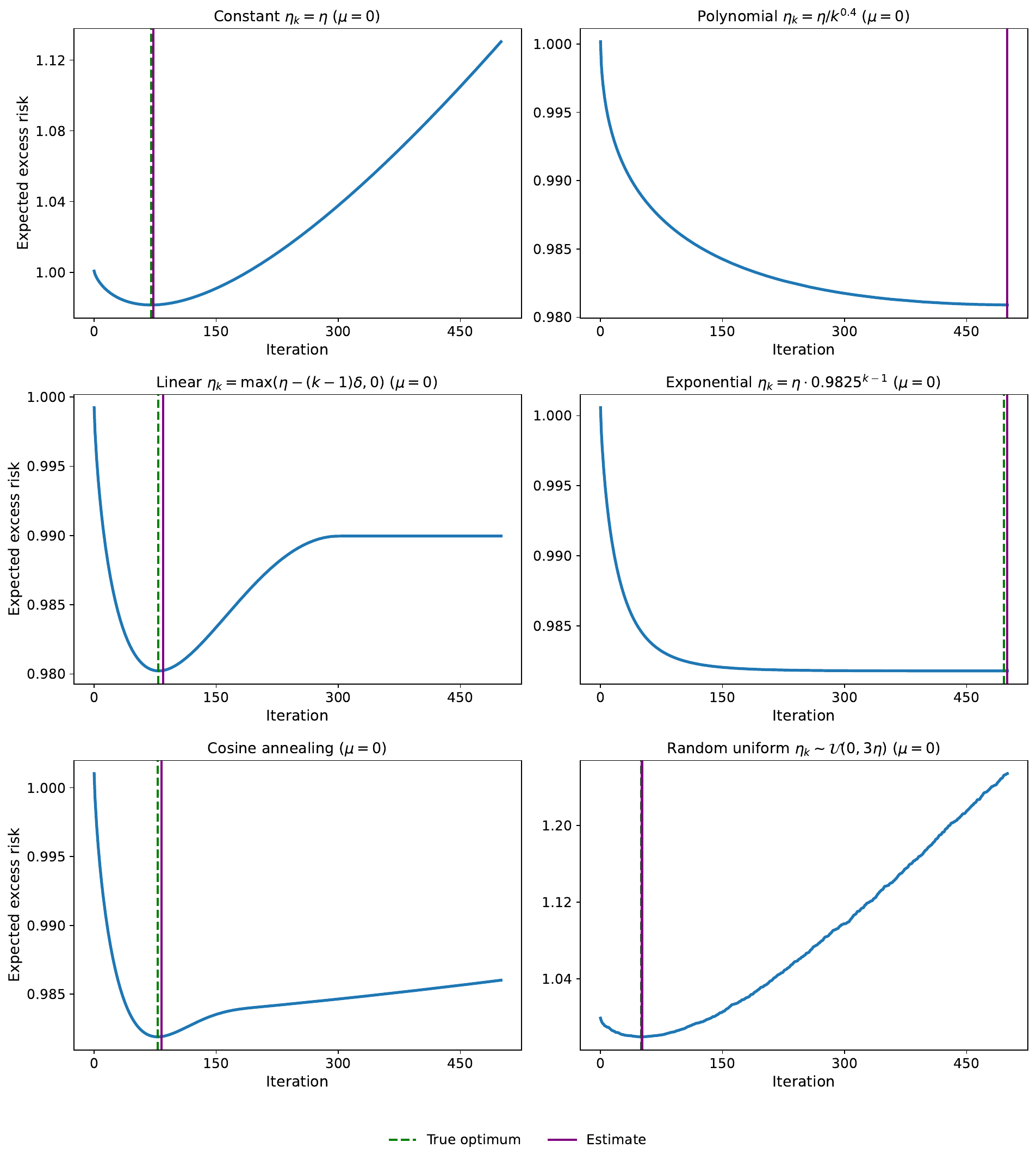}
    \caption{
    Bayes excess risk trajectories for the unregularized least squares problem ($\mu = 0$) with the CIFAR10 dataset. The six panels compare different learning rate schedules. The green dashed line shows the empirical optimal stopping time, and the purple line shows our estimated stopping time~\eqref{eq:when-to-stop-k}.
    }
    \label{fig:real-cifar-mu0}
\end{figure}

\begin{figure}[!htb]
    \centering
    \includegraphics[width=\textwidth, trim={0 1.5cm 0 0}, clip]{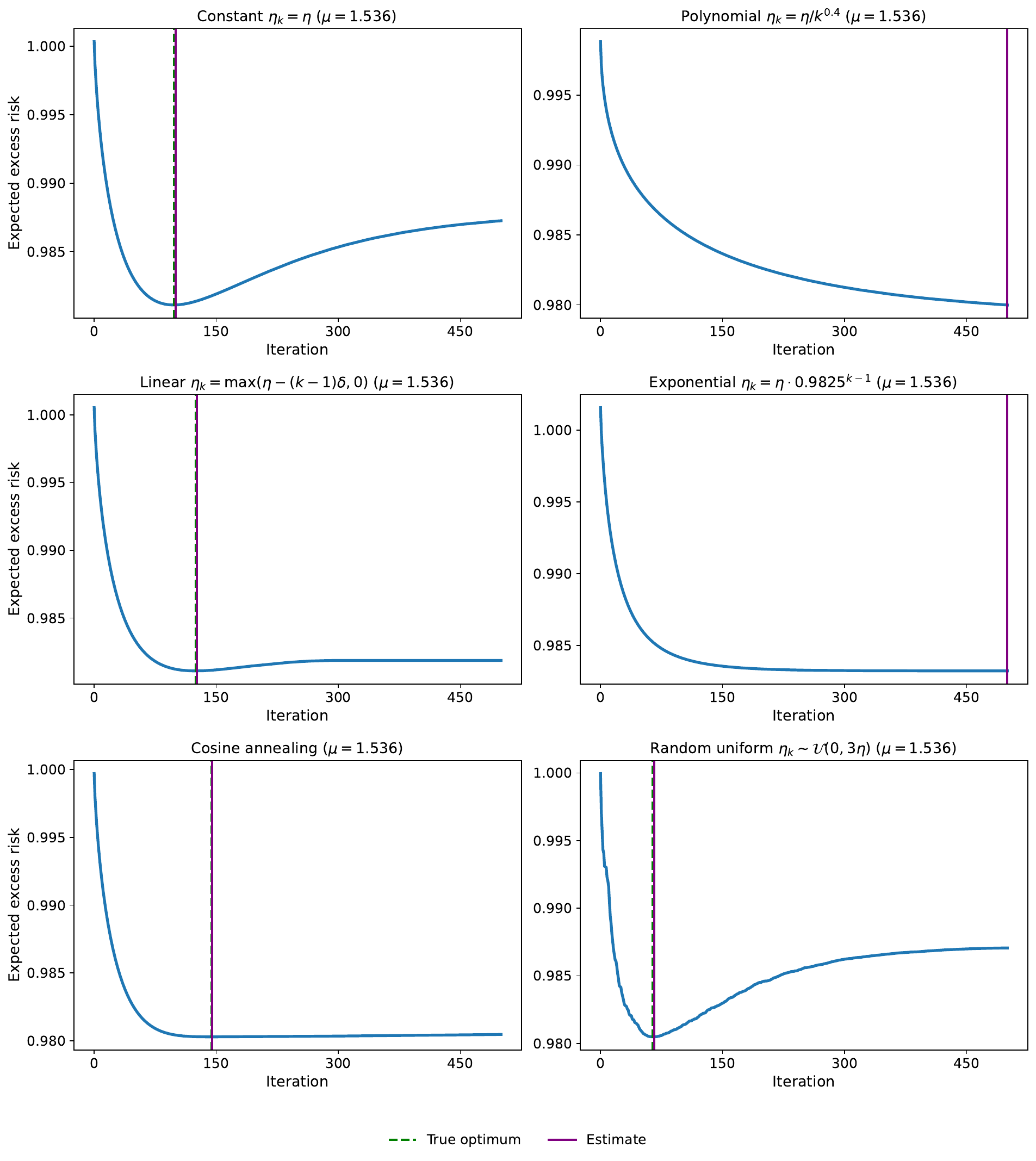}
    \caption{
    Bayes excess risk trajectories for the ridge-regularized least squares problem with $\mu = 1.536$ with the CIFAR10 dataset. This corresponds to $\mu = \mu^* / 2$, where $\mu^* = 3.072$ is the optimal ridge parameter. The six panels compare different learning rate schedules. The green dashed line shows the empirical optimal stopping time, and the purple line shows our estimated stopping time~\eqref{eq:when-to-stop-k}.
    }
    \label{fig:real-cifar-mu-half}
\end{figure}

\begin{figure}[!htb]
    \centering
    \includegraphics[width=\textwidth]{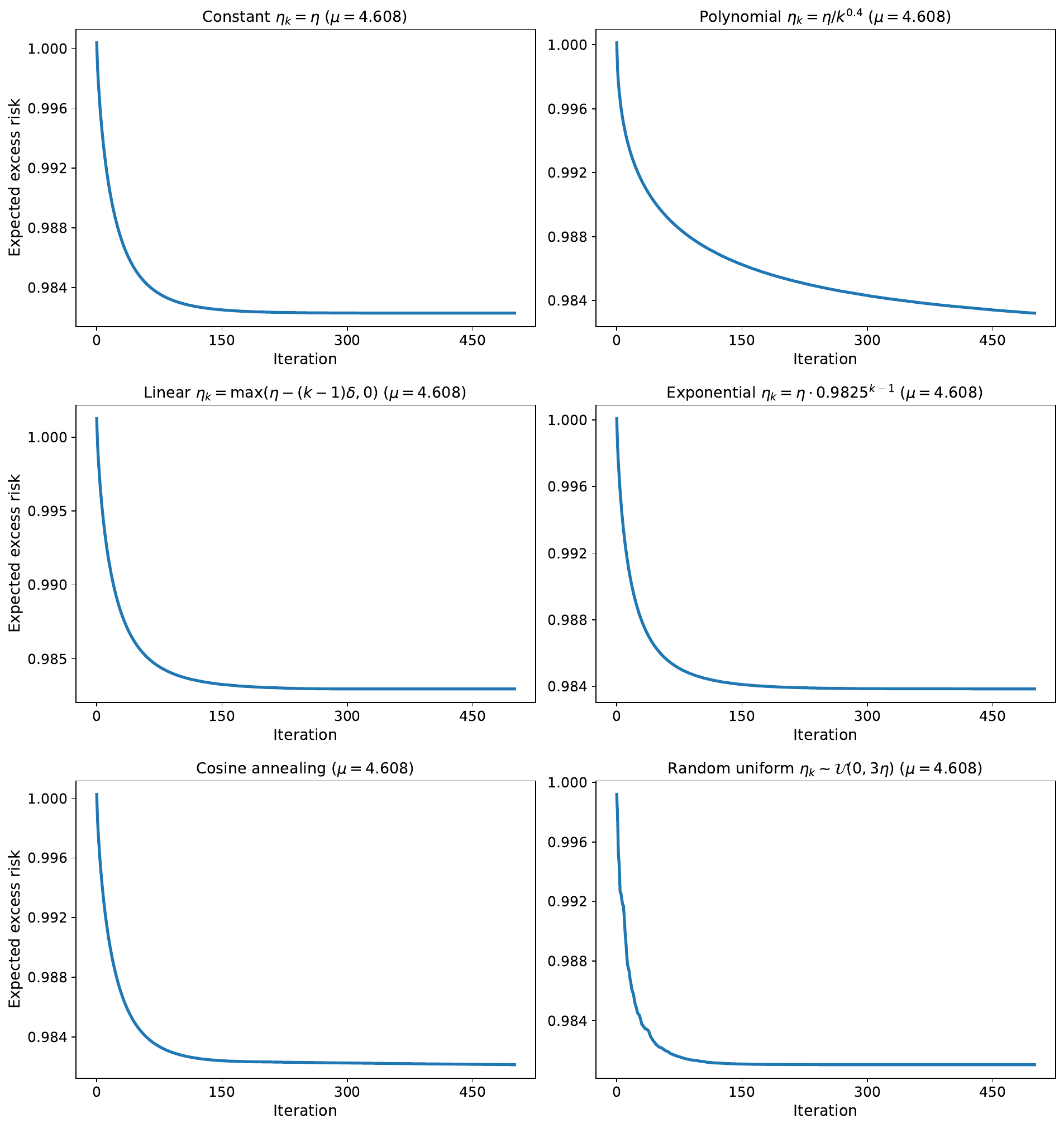}
    \caption{
    Bayes excess risk trajectories for the ridge-regularized least squares problem with $\mu = 4.608$ with the CIFAR10 dataset. This corresponds to $\mu = 1.5\mu^*$, where $\mu^* = 3.072$ is the optimal ridge parameter. The six panels compare different learning rate schedules. Since $\mu > \mu^*$, the theory predicts that early stopping is not beneficial, and accordingly no lines for the optimal or estimated stopping time are plotted.
    }
    \label{fig:real-cifar-mu-large}
\end{figure}

\paragraph{Violating assumptions.}
The random learning rate schedule provides a useful stress test for the assumptions behind the theory. In particular, it should not satisfy the monotone differentiability condition in Assumption~\ref{assumption:differentiable-extension}. We see that the risk curves with random learning rates are less smooth. Nevertheless, the qualitative behavior is still largely governed by the cumulative budget $\sum_{i=1}^k \eta_i$.

We also clarify that if either Assumption~\ref{assumption:noise} or~\ref{assumption:spherical} is violated, then the exact trajectory formula (Proposition~\ref{prop:betak}) and the equivalence to generalized ridge regression (Theorems~\ref{thm:equivalence1} and~\ref{thm:equivalence2}) still hold. However, the estimate~\eqref{eq:when-to-stop-k} the optimal stopping time may not be accurate.

\nocite{pytorch2019}

\clearpage

\section{Future work and conclusions}

In this paper, we analyzed the discrete dynamics of gradient descent for linear regression with generic data and learning rate schedules. By determining expressions for the exact trajectory of the parameters, we provided various results that formalize the intuition that early stopping is similar to $L_2$ regularization. Furthermore, we established general conditions on the learning rate and spectrum of the sample covariance of the feature matrix that show whether early stopping is beneficial or not. Finally, we provided an estimate for the optimal stopping time, which we verified empirically.

An important direction for future work is to extend the results to the non-linear case. While the current work is limited to linear models, we believe that our approach can serve as a foundation for more complex models, such as multi-layer neural networks.
For example, a simple case that may be tractable is the analysis of two-layer neural networks with layer-wise training. Existing works such as~\cite{moniri2024theory, wang2024nonlinear} describe how the spectrum of the features learned by the first layer evolves during training. Once the spectrum is known, our framework can be used to analyze the risk for the overall model by introducing two time parameters $T_1$ and $T_2$, corresponding to the number of steps each layer is trained for, which can then be optimized to minimize the risk. Understanding the dynamics of jointly training both layers is a significant and challenging open problem.

Another direction is to understand how to choose the learning rate schedule itself. In this paper, we study the effect of early stopping given  a fixed schedule, and our stopping time estimate identifies when that schedule has accumulated an appropriate amount of learning. However, different schedules can move through the optimization path in different ways, and they may not always lead to the same optimal risk. A natural next step is to develop criteria for comparing different schedules, for example by optimizing over families of schedules or by using the exact risk formula to choose schedules that balance statistical performance and computational cost.

Furthermore, it would also be of interest to analyze early stopping for other training algorithms, such as stochastic gradient descent. For example, recent work~\cite{LokSonthaliaRebrova2025} shows that the dynamics of mini-batch gradient descent evolves in an analogous way as full-batch gradient descent, depending on a modified covariance matrix that encodes dependencies between the mini-batches.
Another natural direction is to extend the theory to adaptive methods. The results in this paper apply to learning rate schedules that can be treated as fixed after conditioning on the data, independent of the optimization trajectory. However, many adaptive methods choose step sizes using the current iterates, gradients, or residuals. Extending the analysis to such adaptive schedules, and further to coordinate-wise methods such as AdaGrad and Adam, would be a natural direction for future work.

Finally, our analysis shows that the optimal stopping time is, in general, different for each eigendirection of the covariance of the features. This is impossible to implement in practice with the usual gradient descent algorithm. However, it would be interesting to study principled methods that augment the learning dynamics in the later stages so as to enhance the movement along the directions with larger optimal stopping times.

\section*{Acknowledgments}

This work is supported in part by funds from the National Science Foundation NSF DMS \#2309685 (ER and JL). The authors thank the anonymous reviewers for their careful reading and helpful suggestions that have substantially improved the presentation of the paper.

\printbibliography

\appendix

\section{Deferred proofs}
\label{app:deferred_proofs}

In this section, we will provide the proofs for the results that have been deferred, namely Proposition~\ref{prop:betak}, Proposition~\ref{prop:gen-risk-betak-mu}, and Theorem~\ref{thm:earlystopping-mu}.

\subsection{Proof of Proposition~\ref{prop:betak}.}

First, we will prove an identity for sums involving the step sizes $\eta_i$ and the function $\phi$ from Definition~\ref{defn:phi}.

\begin{lemma}[Technical sum identity] \label{lem:sum}
For any non-zero real number $\zeta$, we have
\[
    \sum_{i=1}^k \eta_i \frac{1}{\phi(i; \zeta)} = \frac{1}{\zeta} \left( \frac{1}{\phi(k; \zeta)} - \frac{1}{\phi(0; \zeta)} \right).
\]
\end{lemma}

\begin{proof}[Proof.]
For notational simplicity, we will write $\phi(i) = \phi(i; \zeta)$. We will prove the identity using induction.
For the base case with $k = 1$, the left hand side is $\eta_1 / \phi(1)$, and the right hand side is equal to
\[
    \frac{1}{\zeta}\left(\frac{1}{\phi(1)} - \frac{1}{\phi(0)}\right) = \frac{1}{\zeta}\left(\frac{1 - (1-\eta_1 \zeta}){\phi(1)}\right) = \frac{\eta_1}{\phi(1)}.
\]
Hence, the base case holds.
For the inductive step, assuming that the identity holds for some $k$, we have
\begin{multline*}
    \sum_{i=1}^{k+1}\eta_i \frac{1}{\phi(i)}
    = \frac{\eta_{k+1}}{\phi(k+1)} + \sum_{i=1}^k \eta_i \frac{1}{\phi(i)}
    = \frac{\eta_{k+1}}{\phi(k+1)} + \frac{1}{\zeta}\left(\frac{1}{\phi(k)} - \frac{1}{\phi(0)}\right) \\
    = \frac{\eta_{k+1} \zeta}{\phi(k+1) \zeta} + \frac{(1 - \eta_{k+1} \zeta)}{\phi(k+1) \zeta} - \frac{1}{\zeta \phi(0)}
    = \frac{1}{\zeta}\left(\frac{1}{\phi(k+1)} - \frac{1}{\phi(0)}\right).
\end{multline*}
This shows that the identity also holds for $k+1$. This completes the proof.
\end{proof}

\begin{proof}[Proof of Proposition~\ref{prop:betak}.]
From~\eqref{eq:update}, the gradient update is given by
\begin{align*}
    \beta_{k+1} 
    &= \beta_k - \frac{\eta_{k+1}}{n}X^T(X\beta_k - X\beta_* - \varepsilon) - \eta_{k+1}\mu \beta_k.
\end{align*}
Subtracting $\beta_*$ from both sides, we obtain
\begin{align*}
    \beta_{k+1} - \beta_* &= (1-\eta_{k+1}\mu)(\beta_k - \beta_*) - \frac{\eta_{k+1}}{n}X^TX(\beta_{k} - \beta_*) +  \frac{\eta_{k+1}}{n}(X^T\varepsilon - \mu n \beta_*)\\
    &= \left[(1-\eta_{k+1}\mu)I - \frac{\eta_{k+1}}{n}X^TX\right](\beta_{k} - \beta_*) +  \frac{\eta_{k+1}}{n}\left(X^T\varepsilon - \mu n \beta_*\right)\\
    &= V\left[(1-\eta_{k+1}\mu)I - \frac{\eta_{k+1}}{n}\Lambda\right]V^T(\beta_k - \beta_*) +  \frac{\eta_{k+1}}{n}V\left(\Sigma_X^TU^T\varepsilon - \mu n V^T\beta_*\right).
\end{align*}
By multiplying both sides by $V^T$ to perform a change of basis and recalling the notations $\tilde{\beta}_k = V^T \beta_k$, $\tilde{\beta}_* = V^T\beta_*$, and $\check{\varepsilon} = U^T \varepsilon$ from~\eqref{eq:in-eigbasis}, we have
\[
    \tilde{\beta}_{k+1} - \tilde{\beta}_* = \left[(1-\eta_{k+1}\mu)I - \frac{\eta_{k+1}}{n}\Lambda\right](\tilde{\beta}_k-\tilde{\beta}_*) + \frac{\eta_{k+1}}{n}\left(\Sigma_X^T\check{\varepsilon} - \mu n \tilde{\beta}_*\right).
\]
We can rewrite this as the matrix equation 
\[
    \begin{bmatrix} \tilde{\beta}_{k+1} - \tilde{\beta}_* \\ \Sigma_X^T\check{\varepsilon} - \mu n \tilde{\beta}_* 
    \end{bmatrix} = \begin{bmatrix} \left[(1-\eta_{k+1}\mu)I - \frac{\eta_{k+1}}{n}\Lambda\right] & \frac{\eta_{k+1}}{n}I \\ 0 & I \end{bmatrix} \begin{bmatrix} \tilde{\beta}_{k} - \tilde{\beta}_* \\ \Sigma_X^T\check{\varepsilon} - \mu n \tilde{\beta}_*
    \end{bmatrix}.
\]
Note that the product of block upper triangular matrices is also block upper triangular.
Thus, we have
\begin{align*}
    \begin{bmatrix} \tilde{\beta}_{k+1} - \tilde{\beta}_* \\ \Sigma_X^T\check{\varepsilon}  - \mu n \tilde{\beta}_* 
    \end{bmatrix}
    &= \prod_{i=1}^{k+1}\begin{bmatrix} \left[(1-\eta_{i}\mu)I - \frac{\eta_{i}}{n}\Lambda\right] & \frac{\eta_{i}}{n}I \\ 0 & I \end{bmatrix} \begin{bmatrix} \tilde{\beta}_{0} - \tilde{\beta}_* \\ \Sigma_X^T\check{\varepsilon}  - \mu n \tilde{\beta}_*
    \end{bmatrix} \\
    &= \begin{bmatrix}\prod_{i=1}^{k+1}\left[(1-\eta_{i}\mu)I - \frac{\eta_{i}}{n}\Lambda\right] & \sum_{i=1}^{k+1}\left(\prod_{j=i+1}^{k+1} \left[(1-\eta_{j}\mu)I - \frac{\eta_{j}}{n}\Lambda\right]\right)\frac{\eta_{i}}{n}I \\ 0 & I \end{bmatrix} \begin{bmatrix} \tilde{\beta}_{0} - \tilde{\beta}_* \\ \Sigma_X^T\check{\varepsilon} - \mu n \tilde{\beta}_*
    \end{bmatrix}.
\end{align*}
With the notations from Definition~\ref{defn:phi}, from \eqref{eq:defn_Phi_matrix2}, we have 
\[
    \begin{bmatrix} \tilde{\beta}_{k+1} - \tilde{\beta}_* \\ \Sigma_X^T\check{\varepsilon} - \mu n \tilde{\beta}_* 
    \end{bmatrix}
    = \begin{bmatrix} \Phi(k+1; \mu, \Lambda) & \sum_{i=1}^{k+1}\left(\prod_{j=i+1}^{k+1} \left[(1-\eta_{j}\mu)I - \frac{\eta_{j}}{n}\Lambda\right]\right)\frac{\eta_{i}}{n}I \\ 0 & I \end{bmatrix} \begin{bmatrix} \tilde{\beta}_{0} - \tilde{\beta}_* \\ \Sigma_X^T\check{\varepsilon} - \mu n \tilde{\beta}_*
    \end{bmatrix}.
\]
We can rewrite the sum in the top right block as follows. For the $\ell$th diagonal entry, writing $a_\ell := \mu + n^{-1}\Lambda_\ell$ for brevity, we have
\begin{align}\label{eq:phis}
    \frac{1}{n}\sum_{i=1}^{k+1}\eta_i\left(\prod_{j=i+1}^{k+1} \left[1 - \eta_{j}(\mu + n^{-1}\Lambda_{\ell})\right]\right)
    &= \frac{1}{n}\sum_{i=1}^{k+1}\eta_i\left(\frac{\prod_{j=1}^{k+1} \left[1 - \eta_{j}(\mu + n^{-1}\Lambda_{\ell})\right]}{\prod_{j=1}^{i} \left[1 - \eta_{j}(\mu + n^{-1}\Lambda_{\ell}) \right]}\right) \nonumber\\
    &= \frac{1}{n}\sum_{i=1}^{k+1}\eta_i\left(\frac{\phi(k+1; a_\ell)}{\phi(i; a_\ell)}\right).
\end{align}
If $a_{\ell} \ne 0$, then using the identity from Lemma~\ref{lem:sum}, we have
\begin{align*}
    \frac{1}{n}\sum_{i=1}^{k+1}\eta_i\left(\frac{\phi(k+1; a_\ell)}{\phi(i; a_\ell)}\right)
    &= \frac{1}{n} \frac{1}{a_\ell}\left(1 - \frac{\phi(k+1; a_\ell)}{\phi(0; a_\ell)}\right).
\end{align*}
If $a_{\ell} = 0$, then $\mu = \Lambda_{\ell} = (\Sigma_X^T \Sigma_X)_\ell = 0$, and the corresponding entry of the vector $(\Sigma_X^T\check{\varepsilon} - \mu n \tilde{\beta}_*)_{\ell}$ is also zero.
Thus, we obtain
\[
    \begin{bmatrix} \tilde{\beta}_{k+1} - \tilde{\beta}_* \\ \Sigma_X^T\check{\varepsilon} - \mu n \tilde{\beta}_* 
    \end{bmatrix} = \begin{bmatrix} \Phi(k+1; \mu, \Lambda) & \left(\mu n I + \Lambda\right)^{\dagger}(I - \Phi(k+1; \mu, \Lambda))\\ 0 & I \end{bmatrix} \begin{bmatrix} \tilde{\beta}_{0} - \tilde{\beta}_* \\ \Sigma_X^T\check{\varepsilon} - \mu n \tilde{\beta}_*
    \end{bmatrix}.
\]
This implies that~\eqref{eq:betak_2} holds.

Now, we shall show how this implies~\eqref{eq:betak_1}.
Since $\Lambda$ and $\Phi(k; \mu)$ are diagonal, we have
\begin{align*}
    \tilde{\beta}_k
    &\overset{\eqref{eq:betak_2}}{=} \Phi(k; \mu)(\tilde{\beta}_0 - \tilde{\beta}_*) + (I - \Phi(k; \mu)) (\mu n I + \Lambda)^{\dagger} (\Sigma_X^T \check{\varepsilon} - \mu n \tilde{\beta}_*) +  \tilde{\beta}_*\\
    &= \Phi(k; \mu)(\tilde{\beta}_0 - \tilde{\beta}_*) + (I - \Phi(k; \mu)) (\mu n I + \Lambda)^{\dagger} (\Sigma_X^T \check{\varepsilon} + \Lambda \tilde{\beta}_*) \\
    &\quad\quad\quad\quad\quad\quad\quad\quad- (I - \Phi(k; \mu)) (\mu n I + \Lambda)^{\dagger} (\mu n I + \Lambda) \tilde{\beta}_* + \tilde{\beta}_* \\
    &= \Phi(k; \mu) \tilde{\beta}_0 + (I - \Phi(k; \mu)) (\mu n I + \Lambda)^{\dagger} (\Sigma_X^T \check{\varepsilon} + \Lambda \tilde{\beta}_*).
\end{align*}
The last step holds since $\Lambda$ is a diagonal matrix that is non-zero in its $r \times r$ leading principal submatrix, where $\mathrm{rank}(X) = r$, and $\Phi(k; \mu)$ is diagonal with diagonal entries $\Phi_j(k; \mu) = 1$ for $j > r$, so that
\[
    (I - \Phi(k; \mu)) (\mu n I + \Lambda)^{\dagger} (\mu n I + \Lambda) = I - \Phi(k; \mu).
\]
By multiplying by $V$ to change back to the original basis, recalling that $V \Lambda V^T = X^T X$, we have
\[
    \beta_k
    = V \Phi(k; \mu) V^T \beta_0 + (I - V \Phi(k; \mu) V^T) (\mu n I + X^T X)^{\dagger} X^T(X \beta_* + \varepsilon).
\]
Since $y = X\beta_* + \varepsilon$, this implies~\eqref{eq:betak_1}. This completes the proof.
\end{proof}

\subsection{Proof of Proposition~\ref{prop:gen-risk-betak-mu}}

\begin{proof*}
From Equation~\eqref{eq:betak_2} of Proposition~\ref{prop:betak}, we have that
\begin{align} 
    \tilde{\beta}_k - \tilde{\beta}_*
    &= \Phi(k; \mu)(\tilde{\beta}_0 - \tilde{\beta}_*) + (\mu n I + \Lambda)^{\dagger}(I - \Phi(k; \mu))(\Sigma_X^T \check{\varepsilon} - \mu n \tilde{\beta}_*) \nonumber\\
    &= \underbrace{\Phi(k; \mu)(\tilde{\beta}_0 - \tilde{\beta}_*) - \mu n (\mu n I + \Lambda)^{\dagger}  (I - \Phi(k; \mu)) \tilde{\beta}_*}_{=: A} + \underbrace{(\mu n I + \Lambda)^{\dagger} (I - \Phi(k; \mu)) \Sigma_X^T U^T \varepsilon}_{=: B}. \label{eq:gen-risk-betak-mu_1}
\end{align}
Our goal is to compute the expected excess risk:
\begin{align*}
    R(\beta_k)
    = \mathbb{E}_\varepsilon[\mathcal{R}(\beta_k) \mid X]
    = \mathbb{E}_\varepsilon \left[ \left\| \Sigma^{1/2} (\beta_k - \beta_*) \right\|_2^2 \mid X \right]
    = \mathbb{E}_\varepsilon \left[ \left\| \Sigma^{1/2} V \left( \tilde{\beta}_k - \tilde{\beta}_* \right) \right\|_2^2 \mid X \right].
\end{align*}
By substituting~\eqref{eq:gen-risk-betak-mu_1} and expanding the square, the expected squared norm can be written
\begin{equation} \label{eq:gen-risk-betak-mu_2}
\begin{aligned}
    \mathbb{E}_\varepsilon \left[ \left\| \Sigma^{1/2} V (\tilde{\beta}_k - \tilde{\beta}_*) \right\|_2^2 \mid X \right]
    &= \mathbb{E}_{\varepsilon} \left[ \left\| \Sigma^{1/2} V A \right\|_2^2 \mid X \right]
    + \mathbb{E}_\varepsilon \left[ \left\| \Sigma^{1/2} V B \right\|_2^2 \mid X \right]
    + 2 \mathbb{E}_\varepsilon \left[ A^T V^T \Sigma V B \mid X \right].
\end{aligned}
\end{equation}
Recall that $X = U \Sigma_X V^T$ and $\Lambda = \Sigma_X^T \Sigma_X$ is diagonal. Thus, $V \Lambda V^T = X^T X$, and from~\eqref{eq:defn_Phi_matrix2},
\[
    V \Phi(k; \mu) V^T = \prod_{i=1}^k (I - \eta_i (\mu I + n^{-1} X^T X)).
\]
The significance of this observation is that $V \Phi(k; \mu) V^T$ is a function of the feature matrix $X$. Hence, using the fact that $V$ is orthogonal (i.e., $V^T V = V V^T = I$), we can write
\begin{align}
    V A &= [V \Phi(k; \mu) V^T] (\beta_0 - \beta_*) - \mu n [V (\mu n I + \Lambda)^{\dagger} V^T] [V (I - \Phi(k; \mu)) V^T] \beta_* \nonumber \\
    &= [V \Phi(k; \mu) V^T] (\beta_0 - \beta_*) - \mu n (\mu n I + X^T X)^{\dagger} (I - V \Phi(k; \mu) V^T) \beta_* \label{eq:gen-risk-betak-mu_VA}
\end{align}
Hence, for the first term of~\eqref{eq:gen-risk-betak-mu_2}, we have $\mathbb{E}_{\varepsilon}\left[ \| \Sigma^{1/2} V A \|_2^2 \mid X \right] = \| \Sigma^{1/2} V A \|_2^2$ since $\Sigma^{1/2} VA$ is a function of $X$.
Similarly, to compute the second term of~\eqref{eq:gen-risk-betak-mu_2}, we can write
\begin{align}
    VB &= [V(\mu n I + \Lambda)^{\dagger} V^T][V (I - \Phi(k; \mu)) V^T] [V \Sigma_X^T U^T] \varepsilon \nonumber \\
    &= (\mu n I + X^T X)^{\dagger} (I - V \Phi(k;\mu) V^T) X^T \varepsilon. \label{eq:gen-risk-betak-mu_VB}
\end{align}
Also,
\begin{equation} \label{eq:gen-risk-betak-mu_VB2}
    V B B^T V^T = (\mu n I + X^T X)^{\dagger} (I - V \Phi(k;\mu) V^T) X^T \varepsilon \varepsilon^T X (I - V \Phi(k;\mu) V^T) (\mu n I + X^T X)^{\dagger}.
\end{equation}
Since the residual has covariance matrix $\mathbb{E}_{\varepsilon}[ \varepsilon \varepsilon^T \mid X] = \tau^2 I$ by Assumption~\ref{assumption:noise}, it follows from using the linearity of expectation and~\eqref{eq:gen-risk-betak-mu_VB2} that
\begin{align*}
    &\mathbb{E}[ \mathrm{Tr}(\Sigma V B B^T V^T) \mid X ] \\
    &\quad\quad= \mathrm{Tr}( \Sigma (\mu n I + X^T X)^{\dagger} (I - V \Phi(k;\mu) V^T) X^T \mathbb{E}\left[ \varepsilon \varepsilon^T \mid X \right] X (I - V \Phi(k;\mu) V^T) (\mu n I + X^T X)^{\dagger} ) \\
    &\quad\quad= \tau^2 \mathrm{Tr}( \Sigma (\mu n I + X^T X)^{\dagger} (I - V \Phi(k;\mu) V^T) X^T X (I - V \Phi(k;\mu) V^T) (\mu n I + X^T X)^{\dagger} ) \\
    &\quad\quad= \tau^2 \mathrm{Tr}( \Sigma V (\mu n I + \Lambda)^{\dagger} (I - \Phi(k;\mu)) \Lambda (I - \Phi(k;\mu)) (\mu n I + \Lambda)^{\dagger} V^T).
\end{align*}
Hence, by using the cyclic property of trace, the second term of~\eqref{eq:gen-risk-betak-mu_2} is equal to
\begin{align*}
    \mathbb{E}[\| \Sigma^{1/2} V B \|_2^2 \mid X]
    &= \mathbb{E}[ \mathrm{Tr}( B^T V^T \Sigma V B ) \mid X ]
    = \mathbb{E}[ \mathrm{Tr}(\Sigma V B B^T V^T) \mid X ] \\
    &= \tau^2 \mathrm{Tr}( \Sigma^{1/2} V (\mu n I + \Lambda)^{\dagger} (I - \Phi(k;\mu)) \Sigma_X^T \Sigma_X (I - \Phi(k;\mu)) (\mu n I + \Lambda)^{\dagger} V^T \Sigma^{1/2}) \\
    &= \tau^2 \left\| \Sigma^{1/2} V (\mu n I + \Lambda)^{\dagger}  (I - \Phi(k; \mu)) \Sigma_X^T \right\|_F^2.
\end{align*}
Finally, for the third term of~\eqref{eq:gen-risk-betak-mu_2}, using the observation~\eqref{eq:gen-risk-betak-mu_VA} that $VA$ is a function of $X$, and the expression~\eqref{eq:gen-risk-betak-mu_VB} for $VB$, we have
\begin{align*}
    \mathbb{E}_\varepsilon \left[ A^T V^T \Sigma V B \mid X \right]
    &= \mathbb{E}_\varepsilon \left[ A^T V^T \Sigma (\mu n I + X^T X)^{\dagger} (I - V \Phi(k;\mu) V^T) X^T \varepsilon \mid X \right] \\
    &= A^T V^T \Sigma (\mu n I + X^T X)^{\dagger} (I - V \Phi(k;\mu) V^T) X^T \mathbb{E}_\varepsilon \left[ \varepsilon \mid X \right] = 0,
\end{align*}
since $\mathbb{E}_{\varepsilon}[\varepsilon \mid X] = 0$ by Assumption~\ref{assumption:noise}. Thus, the cross-term vanishes, and we conclude that $R(\beta_k) = \| \Sigma^{1/2} VA \|_2^2 + \mathbb{E}[\| \Sigma^{1/2} V B \|_2^2 \mid X]$, which leads to the claimed result.
\end{proof*}

\subsection{Proof of Theorem~\ref{thm:earlystopping-mu}}

\begin{proof*}
Let $P = V^T \Sigma V$ as in the proof of Theorem~\ref{thm:earlystopping}. From Proposition~\ref{prop:gen-risk-betak-mu}, the expected excess risk after $k$ iterations is given by
\begin{equation} \label{eq:earlystopping-mu-pf:1}
    R(\beta_k) = \gamma^T P \gamma + \tau^2 \Tr\left( P \left[(\mu n I + \Lambda)^\dagger\right]^2 (I - \Phi(k; \mu))^2 \Lambda\right),
\end{equation}
where 
\[
    \gamma := \left[ \Phi(k; \mu)(- \tilde{\beta}_*) - \mu n (\mu n I + \Lambda)^\dagger (I - \Phi(k; \mu)) \tilde{\beta}_* \right].
\]
By expanding the inner product, we can write the first term of~\eqref{eq:earlystopping-mu-pf:1} as the sum of the following three terms:
\begin{align*}
    \gamma^T P \gamma
    &= \left[\Phi(k; \mu)(- \tilde{\beta}_*)\right]^T P \Phi(k; \mu)(- \tilde{\beta}_*) \\
    &\quad + \left[\mu n (\mu n I + \Lambda)^\dagger (I - \Phi(k; \mu)) \tilde{\beta}_*\right]^T P \mu n (\mu n I + \Lambda)^\dagger (I - \Phi(k; \mu)) \tilde{\beta}_* \\
    &\quad -2\left[\mu n (\mu n I + \Lambda)^\dagger (I - \Phi(k; \mu)) \tilde{\beta}_*\right]^T P\Phi(k; \mu)( - \tilde{\beta}_*).
\end{align*}
Taking expectations with respect to $\beta_*$ under Assumption~\ref{assumption:spherical}, we see that the first summand becomes
\[
    \sum_{j=1}^{p} P_{jj} \cdot \sigma^2 \Phi_j(k)^2,
\]
where we denote the $j$th diagonal entry of $\Phi(k;\mu)$ by the shorthand $\Phi_j(k) = \phi(k;\mu + n^{-1} \Lambda_j) / \phi(0;\mu + n^{-1} \Lambda_j)$.
Similarly, the second summand becomes
\[
     \sum_{j=1}^{p} P_{jj}\cdot \sigma^2 \frac{\mu^2 n^2}{(\mu n + \Lambda_j)^2} (1 - \Phi_j(k))^2,
\]
and the third summand becomes
\[
    2 \sum_{j=1}^{p} P_{jj} \cdot \sigma^2 \frac{\mu n}{\mu n + \Lambda_j}\Phi_j(k) (1-\Phi_j(k)).
\]
Next, we can write the second term of~\eqref{eq:earlystopping-mu-pf:1} as the following:
\[
    \tau^2 \Tr\left( P \left[(\mu n I + \Lambda)^\dagger\right]^2 (I - \Phi(k; \mu))^2 \Lambda\right) = \sum_{j=1}^{p} P_{jj} \cdot \frac{\tau^2 \Lambda_j}{(\mu n + \Lambda_j)^2}(1-\Phi_j(k))^2.
\]
Thus, summing the displayed expressions above, we have shown that the Bayes excess risk is given by
\[
    R_B(\beta_k) =  \sum_{j=1}^{p} P_{jj} \left[\Phi_j(k)^2 \sigma^2 + \left(\frac{\mu^2 n^2 \sigma^2 + \tau^2 \Lambda_j}{(\mu n + \Lambda_j)^2}\right)(1 - \Phi_j(k))^2 + 2 \sigma^2 \frac{\mu n}{\mu n + \Lambda_j} \Phi_j(k) (1 - \Phi_j(k)) \right].
\]
To analyze the benefit of early stopping, we compute the derivative with respect to $k$:
\begin{align*}
    \partial_k R_B(\beta_k) = 2 \sum_{j=1}^{p} P_{jj} \cdot \partial_k \Phi_j(k) \left[ \Phi_j(k)\sigma^2 - \left(\frac{\mu^2 n^2 \sigma^2 +  \tau^2\Lambda_j}{(\mu n + \Lambda_j)^2}\right)(1-\Phi_j(k)) + \sigma^2\frac{\mu n}{\mu n + \Lambda_j} (1 - 2\Phi_j(k)) \right].
\end{align*}
Note that with a ridge regularization parameter $\mu > 0$, $\Phi_j(k)$ is not necessarily constant for $j > r$, where $r = \mathrm{rank}(X)$.
Since $\partial_k \Phi_j(k) \leq 0$, the sign of $\partial_k R_B(\beta_k)$ depends on the sign of the expression inside the brackets:
\[
   \Phi_j(k)\sigma^2 - \left(\frac{\mu^2 n^2 \sigma^2 +  \tau^2\Lambda_j}{(\mu n + \Lambda_j)^2}\right)(1-\Phi_j(k)) + \sigma^2\frac{\mu n}{\mu n + \Lambda_j} (1 - 2\Phi_j(k)).
\]
Simplifying, the coefficient of the terms with $\Phi_j(k)$ is given by
\begin{align*}
    \sigma^2\left[1 + \frac{\mu^2 n^2 }{(\mu n + \Lambda_j)^2} - 2\frac{\mu n}{\mu n + \Lambda_j} \right] + \frac{ \tau^2\Lambda_j}{(\mu n + \Lambda_j)^2}
    &= \sigma^2 \frac{\Lambda_j^2}{(\mu n + \Lambda_j)^2} + \frac{ \tau^2\Lambda_j}{(\mu n + \Lambda_j)^2} \\
    &= \Lambda_j \cdot \frac{\sigma^2 \Lambda_j + \tau^2}{(\mu n + \Lambda_j)^2},
\end{align*}
and the coefficient for the remaining terms without $\Phi_j(k)$ is given by
\begin{align*}
    - \frac{\mu^2 n^2 \sigma^2 + \tau^2\Lambda_j}{(\mu n + \Lambda_j)^2}+ \sigma^2\frac{\mu n}{\mu n + \Lambda_j}
    &= \Lambda_j \cdot \frac{\sigma^2 \mu n - \tau^2}{(\mu n + \Lambda_j)^2}.
\end{align*}
Thus, noting that $\Lambda_j = 0$ for $j > r$, we have shown that
\begin{align}
    \partial_k R_B(\beta_k)
    &= 2 \sum_{j=1}^{p} P_{jj} \cdot \partial_k \Phi_j(k) \cdot \frac{\Lambda_j}{(\mu n + \Lambda_j)^2} \left[\Phi_j(k) \cdot (\sigma^2 \Lambda_j + \tau^2) + (\sigma^2 \mu n - \tau^2) \right] \nonumber\\
    &= 2 \sum_{j=1}^{r} P_{jj} \cdot \partial_k \Phi_j(k) \cdot \frac{\Lambda_j}{(\mu n + \Lambda_j)^2} \left[\Phi_j(k) \cdot (\sigma^2 \Lambda_j + \tau^2) + (\sigma^2 \mu n - \tau^2) \right]. \label{eq:risk-derivative}
\end{align}
Observe that for each $j$, solving for the value of $\Phi_j(k)$ for which the square bracket is zero yields
\[
    \Phi_j(k) = \frac{\tau^2 - \sigma^2 \mu n}{\tau^2 + \sigma^2 \Lambda_j}.
\]
Therefore, if the condition~\eqref{eq:earlystopping-mu_condition1} holds for all $j = 1, \ldots, r$, then the gradient is eventually positive. That is, for sufficiently large $k$, $\partial_k R_B(\beta_k) \geq 0$, which implies that early stopping is beneficial.
Conversely, if~\eqref{eq:earlystopping-mu_condition2} holds for all $j = 1, \ldots, r$, then we deduce that $\partial_k R_B(\beta_k) \leq 0$ for all $k$, which implies that early stopping is not beneficial.
\end{proof*}

\section{Properties of the \texorpdfstring{$q$}{q}-Pochhammer symbol}
\label{app:q-pochh}

In this section, our goal is to prove some facts about the $q$-Pochhammer symbol $$(a; q)_n = \prod_{i=0}^{n-1} (1 - aq^i).$$ In particular, we want to justify being able to take derivatives in $n$ when analyzing the risk of gradient descent with an exponentially decaying learning rate schedule (see Proposition~\ref{prop:exponential} and Assumption~\ref{assumption:differentiable-extension}).
Since $(a; q)_n$ is only defined for integer $n$ so far, we first need to find an extension to the set of real numbers. Throughout this section, we shall assume that $|q| < 1$.

First, observe that we can write
\begin{equation}
    (a;q)_x = \frac{(a;q)_{\infty}}{(aq^x;q)_\infty}.
\end{equation}
Here we see that the right hand side is defined for all $x \in \mathbb{R}$. Note that this is closely connected to the following known identity for the $q$-Gamma function (e.g., \cite{koekoek1996askey}):
\[
    \Gamma_q(x) = \frac{(q;q)_{\infty}}{(q^x;q)_\infty} (1-q)^{1-x}.
\]
The main result of this section is the following, which shows that this extension of the $q$-Pochhammer symbol to the reals is differentiable.

\begin{proposition} \label{prop:ratio-limit}
Let $|q| < 1$. The function
\[
    x \mapsto (a; q)_x := \frac{(a; q)_\infty}{(a q^x; q)_\infty}
\]
is differentiable in $x$, and the derivative is given by
\[
        \partial_x (a; q)_x = \frac{(a;q)_\infty}{(aq^x;q)_\infty^2} aq^{x}\log(q)\sum_{j=0}^\infty \frac{q^j}{1-aq^{x+j}} (aq^x;q)_\infty.
\]
Moreover, for all $m \in \mathbb{N}$ we have that 
\[
    \lim_{k \to \infty} \frac{\partial_x (a; q)_x |_{x = k}}{\partial_x (a; q)_x |_{x = k+m}} = \frac{1}{q^m}.
\]
\end{proposition}

We shall build towards the proof of Proposition~\ref{prop:ratio-limit} by stating and proving some technical lemmas.

\begin{lemma} \label{lemma:poch-cts}
For $|q| < 1$, the function $a \mapsto (a;q)_\infty$ is continuous.
\end{lemma}

\begin{proof}[Proof.]
Let $\varepsilon > 0$. We shall show continuity at $a_0$. Fix a number $\delta > 0$, which we shall choose later, and let $a$ be any point such that $|a - a_0 | < \delta$. We want to choose $\delta$ such that $|(a;q)_\infty - (a_0;q)_\infty| < \varepsilon$.
To do so, let us define 
\[
    M := \max\left\{ \left(\max_{i,j = 1, \ldots, \infty} \prod_{k=i}^j| (1-aq^k)|\right) \cdot \left(\max_{i,j = 1, \ldots, \infty} \prod_{k=i}^j |(1-a_0q^k)|\right), 1 \right\}.
\]
This is finite because $|q| < 1$. Thus, only finitely many $(1-aq^k)$, $(1-a_0q^k)$ have magnitude greater than one. We claim that for all $n$, if $|a-a_0| < \delta$, then
\begin{equation} \label{eq:poch-cts_1}
    |(a;q)_n - (a_0;q)_n| \le \delta M \sum_{k=0}^{n-1} |q|^k.
\end{equation}
Assuming that this is true for now, then by taking the limit as $n \to \infty$, we obtain 
\[
    |(a;q)_\infty - (a_0;q)_\infty| \le \frac{\delta M }{1-|q|}. 
\]
Hence, choosing $\delta < (1 - |q|) \varepsilon / M$
completes the proof. Finally, we shall now prove the claim~\eqref{eq:poch-cts_1}. Note that  
\begin{align*}
    &|(a;q)_{n+1}- (a_0;q)_{n+1} | \\
    &\quad= |(a;q)_{n}(1-aq^n) - (a_0;q)_n (1-a_0q^n) - (a_0;q)_n(1-aq^n) + (a_0;q)_n(1-aq^n)| \\
    &\quad\le |1-aq^n||(a;q)_n - (a_0;q)_n| + |(a_0;q)_n||a-a_0||q|^n \\
    &\quad\le |1-aq^n||(a;q)_n - (a_0;q)_n| + \delta |(a_0;q)_n||q|^n \\
    &\quad\le |1-aq^n|\Big[ |1-aq^{n-1}| |(a;q)_{n-1} - (a_0;q)_{n-1}| + \delta |(a_0;q)_{n-1}| |q|^{n-1} \Big] + \delta |(a_0;q)_n||q|^n \\
    &\quad\le \ldots \le \delta \sum_{k=0}^n c_k |q|^k,
\end{align*}
where each $c_k$ is of the form
\[
    c_k = \left(\prod_{\ell=i(k)}^{j(k)}| (1-aq^\ell)|\right) \cdot \left(\prod_{\ell=i_0(k)}^{j_0(k)} |(1-a_0q^\ell)|\right).
\]
Thus, we see that for all $k$, $c_k \le M$, which implies the claimed result.
\end{proof}

Now that we have continuity, we want to bootstrap this to get differentiability. To do this, we shall need to prove the following lemma:

\begin{lemma} \label{lemma:locally-uniform}
Let $|q| < 1$. Then for all $a$, $(a;q)_n \to (a;q)_\infty$ locally uniformly as $n \to \infty$.
\end{lemma}

\begin{proof}[Proof.]
Let $a_0$ be any point and $\varepsilon > 0$. We shall prove uniform convergence in the closed ball $B(a_0, \varepsilon) := \{ a : |a - a_0| \le \varepsilon \}$.
Let
\[
    M = \max_{a \in B(a_0, \varepsilon)} |a|.
\]
For any $a \in B(a_0,\varepsilon)$, we see that 
\[
    |(a;q)_n - (a;q)_\infty| = |(a;q)_n||1-(aq^n;q)_\infty|.
\]
By Lemma \ref{lemma:poch-cts}, the function $a \mapsto (a; q)_\infty$ is continuous at zero. Thus, there exists a number $\delta > 0$ such that if $|\tilde{a}| < \delta$, we have that 
\[
    |1 - (\tilde{a};q)_\infty| \le \frac{\varepsilon}{\max_k |(a;q)_k|}.
\]
Note that $\max_k (a;q)_k$ is again finite since $|q| < 1$ and $|a| \le M$.
Thus, by choosing $N$ such that $aq^N < \delta$, which is possible since $|a| \le M$, we obtain
\[
    |(a;q)_n||1-(aq^n;q)_\infty| \le |(a;q)_n| \, \frac{\varepsilon}{\max_k |(a;q)_k|} \le \varepsilon
\]
for all $n \geq N$. Thus, we have locally uniform convergence. 
\end{proof}

\begin{lemma} \label{lemma:derivative}
If $|q| < 1$, then 
\[
    \partial_x (aq^x;q)_\infty = -aq^{x}\log(q)\sum_{k=0}^\infty \frac{q^k}{1-aq^{x+k}} (aq^x;q)_\infty.
\]
\end{lemma}

\begin{proof}[Proof.]
Since we have locally uniform convergence from Lemma~\ref{lemma:locally-uniform}, we can use the formula for the derivative of an infinite product of analytic functions to obtain
\begin{align*}
    \partial_x(aq^x;q)_\infty &= \sum_{k=0}^\infty \partial_x(1-aq^{x+k}) \prod_{j \neq k}(1-aq^{x+j}) \\
    &= -\sum_{k=0}^\infty aq^{x}\log(q)q^k \prod_{j \neq k}(1-aq^{x+j}) \\
    &= -aq^{x}\log(q)\sum_{k=0}^\infty q^k \prod_{j \neq k}(1-aq^{x+j})\\
    &= -aq^{x}\log(q)\sum_{k=0}^\infty  \frac{q^k}{(1-aq^{x+k})}  (aq^x;q)_\infty.
    \qedhere
\end{align*}
\end{proof}

\begin{proof}[Proof of Proposition~\ref{prop:ratio-limit}.]
The differentiability of $(a; q)_x$ and the formula for the derivative follow from Lemma~\ref{lemma:derivative}.
Taking the ratio of the derivatives evaluated at $k$ and $k + m$, we see that 
\[
    \frac{\partial_x (a; q)_x |_{x = k}}{\partial_x (a; q)_x |_{x = k+m}} =  \frac{(aq^{k+m};q)_\infty}{(aq^k;q)_\infty} \frac{1}{q^m} \frac{\sum_{j=0}^\infty \frac{q^j}{1-aq^{k+j}}}{\sum_{j=0}^\infty \frac{q^j}{1-aq^{k+m+j}}}.
\]
We note that for each fixed $j \ge 0$,
\[
    \lim_{k\to\infty}\frac{q^j}{1-aq^{k+m+j}}=q^j,
\]
since $q^{k+m+j} \to 0$. So, for $k$ sufficiently large, $|a q^{k+m+j}| \le 1/2$, and hence
\[
    \left| \frac{q^j}{1-aq^{k+m+j}}\right| \le 2|q|^j.
\]
Since $\sum_{j=0}^{\infty}2|q|^j < \infty$, the dominated convergence theorem gives
\[
    \lim_{k\to\infty}\sum_{j=0}^{\infty} \frac{q^j}{1-aq^{k+m+j}}
    = \sum_{j=0}^{\infty} \lim_{k\to\infty} \frac{q^j}{1-aq^{k+m+j}}
    = \sum_{j=0}^{\infty}q^j
    = \frac{1}{1-q}.
\]
Additionally, we also see that 
\[
    \frac{(aq^k;q)_\infty}{(aq^{k+m};q)_\infty}
    = \frac{(aq^k;q)_\infty}{(a;q)_\infty} \frac{(a;q)_\infty}{(aq^{k+m};q)_\infty}
    = \frac{(a;q)_{k+m}}{(a;q)_{k}}
    = \prod_{j=k+1}^{k+m}(1-aq^{j-1}).
\]
If we take the limit of this expression as $k \to \infty$, we deduce that
\[
    \lim_{k \to \infty} \frac{(aq^{k+m};q)_\infty}{(aq^k;q)_\infty} = 1.
\]
Hence, by combining the individual limits above, we conclude that the ratio of the derivatives evaluated at $k$ and $k + m$ tends to $1 / q^m$ as $k \to \infty$.
\end{proof}

\end{document}